\documentclass[lettersize,journal]{IEEEtran}
\usepackage{amsmath,amsfonts}
\usepackage{amsthm}
\usepackage{algorithmic}
\usepackage{multirow}
\usepackage{booktabs}
\usepackage{bm}
\usepackage{array}
\usepackage{url}
\usepackage{graphicx}

\usepackage{caption}
\usepackage{subcaption}
\newlength\savewidth\newcommand\shline{\noalign{\global\savewidth\arrayrulewidth
		\global\arrayrulewidth 1pt}\hline\noalign{\global\arrayrulewidth\savewidth}}
\newtheorem{theorem}{Theorem}
\begin{document}
\title{Evaluating Roadside Perception for Autonomous Vehicles: Insights from Field Testing}
\author{Rusheng Zhang, Depu Meng,
Shengyin Shen, Tinghan Wang, Tai Karir, 
\\    Michael Maile, Henry X. Liu
\thanks{
This research was partially funded by the U.S. Department of Transportation (USDOT) Advanced Transportation and Congestion Management Technologies Deployment Award (693JJ32150006) and Mcity of the University of Michigan.
\textit{Corresponding authors: 
Rusheng Zhang and Henry X. Liu.}}
\thanks{R. Zhang, D. Meng, and H.X. Liu are with the Department
of Civil and Environmental Engineering, University of
Michigan, Ann Arbor, MI, 48109, USA (email: \{rushengz, depum, henryliu\}@umich.edu)}

\thanks{S. Shen is with the University of Michigan 
Transportation Research
Institude, 2901 Baxer Rd, Ann Arbor, MI, 48109, USA.
(email: shengyin@umich.edu)}

\thanks{T. Wang is with the Department of Mechanical Engineering, University of Michigan,
2350 Hayward St, Ann Arbor, MI 48109, USA.
(email: tinghanw@umich.edu)}

\thanks{T. Karir is with Huron High School, Ann Arbor, Michigan, USA.
(email: tai.karir@gmail.com)}

\thanks{M. Maile is with Ivie Communications, (Email: michael.a.maile@gmail.com)}


\thanks{H. X. Liu is also with Mcity of the University of Michigan.}

}


\maketitle

\begin{abstract}
Roadside perception systems are increasingly crucial in enhancing traffic safety and facilitating cooperative driving for autonomous vehicles. Despite rapid technological advancements, a major challenge persists for this newly arising field: the absence of standardized evaluation methods and benchmarks for these systems. This limitation hampers the ability to effectively assess and compare the performance of different systems, thus constraining progress in this vital field. This paper introduces a comprehensive evaluation methodology specifically designed to assess the performance of roadside perception systems. Our methodology encompasses measurement techniques, metric selection, and experimental trial design, all grounded in real-world field testing to ensure the practical applicability of our approach.

We applied our methodology in Mcity\footnote{\url{https://mcity.umich.edu/}}, a controlled testing environment, to evaluate various off-the-shelf perception systems. This approach allowed for an in-depth comparative analysis of their performance in realistic scenarios, offering key insights into their respective strengths and limitations. The findings of this study are poised to inform the development of industry-standard benchmarks and evaluation methods, thereby enhancing the effectiveness of roadside perception system development and deployment for autonomous vehicles. We anticipate that this paper will stimulate essential discourse on standardizing evaluation methods for roadside perception systems, thus pushing the frontiers of this technology. Furthermore, our results offer both academia and industry a comprehensive understanding of the capabilities of contemporary infrastructure-based perception systems.
\end{abstract}

\begin{IEEEkeywords}
Roadside Perception Systems, Automated Vehicles, Standardized Evaluation Method, Roadside Perception System Benchmarks
\end{IEEEkeywords}

\section{Introduction}

A key component of autonomous driving is the perception system, which enables the vehicle to sense and understand its surroundings. However, the onboard perception system, which relies on sensors mounted on the vehicle, such as cameras, LiDARs and radars, may face limitations in complex scenarios, harsh weather, and lighting conditions, due to occlusions, blind spots, sensor noise and environmental diversity. Therefore, a complementary roadside perception system is needed to enhance the onboard perception and provide more systematic and reliable scene information. Roadside perception leverages sensors installed on the roadside infrastructure to detect and track vehicles and other objects in the region of interest, and communicate the perception results to the onboard system via vehicle-to-infrastructure (V2I) communication technologies \cite{smartroad2021,china2022smart}.

Roadside perception is a rising field that has attracted increasing attention from both academia and industry in recent years. Several off-the-shelf products have been introduced to the market, which provide roadside perception solutions based on camera, radar and LiDAR sensors. For example, Derq \cite{derq2020} offers AI-driven roadside perception software that can detect and predict road users’ behavior and prevent collisions. Ouster \cite{ouster2020}  produces high-resolution digital LiDAR sensors that can be used for roadside perception and V2X communication. These examples represent only a small fraction of the burgeoning market of off-the-shelf products aimed at enhancing roadside perception for autonomous driving. There are numerous other companies, as detailed in \cite{smartroad2021, china2022smart}, contributing innovative solutions in this sector. Together, they illustrate both the immense potential this field holds for the future of autonomous driving.

Despite the burgeoning market of roadside perception systems, there currently exists a notable deficit in this field: the absence of a standardized, fair comparison between these diverse systems. Various off-the-shelf products, each with their unique features and performance claims, are currently operating without a clear means of direct, fair comparison. This could result in an unregulated landscape, where the absence of fair competition hampers innovation and quality assurance in roadside perception systems.

In response to this pressing need, this paper presents an evaluation methodology specifically devised for a systematic assessment of roadside perception systems. The proposed methodology encompasses aspects such as measurement techniques, metrics selection, and experimental trial design, all derived from field testing experience. Importantly, these components are deeply rooted in real-world conditions, making this evaluation methodology not only practical but also readily implementable. The focus on real-world application ensures that our proposed system can serve as a useful tool in advancing fair comparison in this field.

To illustrate the viability of our proposed methodology, we implemented it within Mcity, a controlled yet realistic testing environment \cite{MCity}. We evaluated three off-the-shelf perception systems, providing a practical demonstration of how our methodology enables detailed analysis of system performance. By evaluating these three distinct off-the-shelf perception systems within Mcity, our goal is more than just sharing the outcomes of these assessments. In fact, these evaluations serve as a practical illustration of a standardized methodology for assessing roadside perception systems. Our aim is to make a significant stride towards establishing standardized benchmarks for roadside perception systems. We believe that such an advancement will bring substantial benefits to both industry and academia, providing a solid foundation for further research and development in this sector.

\section{Related Works}

Roadside perception systems have experienced remarkable progress in recent years, driven by the dynamic advancements in autonomous driving technologies \cite{smartroad2021,china2022smart}. These systems enhance the onboard vehicle perception, offering more reliable and precise scene comprehension, especially in challenging scenarios involving complex environments or severe weather and lighting conditions \cite{tsukada2020networked,zhang2022design}. Notably, the academic literature presents a variety of roadside perception systems, utilizing different sensor technologies. This includes systems based on LiDAR \cite{gong2021pedestrian}, and camera-based systems \cite{zhang2022design, zhang2020robust, zou2022real}, each offering distinct advantages in capturing environmental data.

One of the applications of roadside perception systems is cooperative driving, which refers to the coordination and collaboration among vehicles and infrastructure to achieve safer and more efficient traffic flow. Cooperative perception aims to share locally perceived data with other vehicles and roadside infrastructure, and is one of the necessary components of cooperative driving~\cite{han2023collaborative}. This technology can enhance the situational awareness of the vehicles, overcome the limitations of onboard sensors, such as long-range, occlusion, and blind-spot issues \cite{tsukada2020networked,zhang2022design}, and increase the accuracy and robustness of detection results~\cite{su2023uncertainty, lei2022latency, vadivelu2021learning}.

Even as cooperative perception technologies advance and an array of off-the-shelf products emerge, a critical hurdle remains in the field of roadside perception systems: the absence of standardized evaluation methods and benchmarks. Related fields, including autonomous driving and object tracking, have established substantial benchmarks and evaluation metrics that have greatly propelled research and application. For instance, in the domain of autonomous driving, several resources such as the KITTI dataset \cite{geiger2012we}, nuScenes \cite{nuscenes2019}, and the Waymo Open Dataset \cite{sun2020scalability}, have emerged as pivotal benchmarks for academic and industrial progression. These resources have also played a crucial role in the development of standardized metrics to accurately assess the performance of autonomous driving systems. A parallel trend is observed in the sphere of video-based multiple object tracking (MOT), where challenges like MOT16/17/20 \cite{milan2016mot16,dendorfer2020mot20}, and TAO \cite{dave2020tao} have defined rigorous, objective and fair metrics to measure performance. Despite these advancements in related fields, there is a pressing need to formulate or adapt such comprehensive evaluation methodologies specifically for roadside perception systems.

Standardized evaluation of roadside perception systems presents notable complexities, especially compared to traditional benchmarking methodologies used for MOT or onboard detection algorithms. These complexities originate from the variety of sensors employed, such as cameras, LiDAR, and radar, as well as their assorted combinations. Additionally, practical limitations in accessing the internal algorithms of these systems further compound the challenges faced.

\section{Problem Formulation}
\label{sec:problem-formulation}
The assessment process is carried out in Mcity, a controlled testing environment that replicates realistic urban road conditions.
The perception systems are strategically deployed at a chosen intersection within Mcity. The positioning of the sensors is meticulously optimized to suit the unique detection capabilities of each sensor type, ensuring the best possible view of road user movements (Refer to Figure \ref{fig:mcity} for the sensor placement and an aerial view of Mcity).
 The evaluation of vehicle detection is facilitated through the use of two autonomous vehicles, each equipped with Real-Time Kinematic (RTK) Global Positioning System (GPS) technology. The RTK GPS provides highly precise trajectory data of the vehicle, with accuracy down to a few centimeters. This level of precision considerably surpasses the accuracy typically attainable by roadside perception systems and more than adequately meets the requirements for autonomous driving. Therefore, we use these recordings from RTK GPS as the ground truth in our evaluation process.
Regarding pedestrian detection, a portable RTK GPS device will be carried by the pedestrian to capture their precise trajectory, thereby providing the ground truth for pedestrian movement. An experimental vehicle and its setup used in our experiment can be seen in Figure \ref{fig:experiment-vehicle}.

\begin{figure*}[ht]
    \centering
    \begin{subfigure}{.5\textwidth}
        \centering
        \includegraphics[width=.9\linewidth]{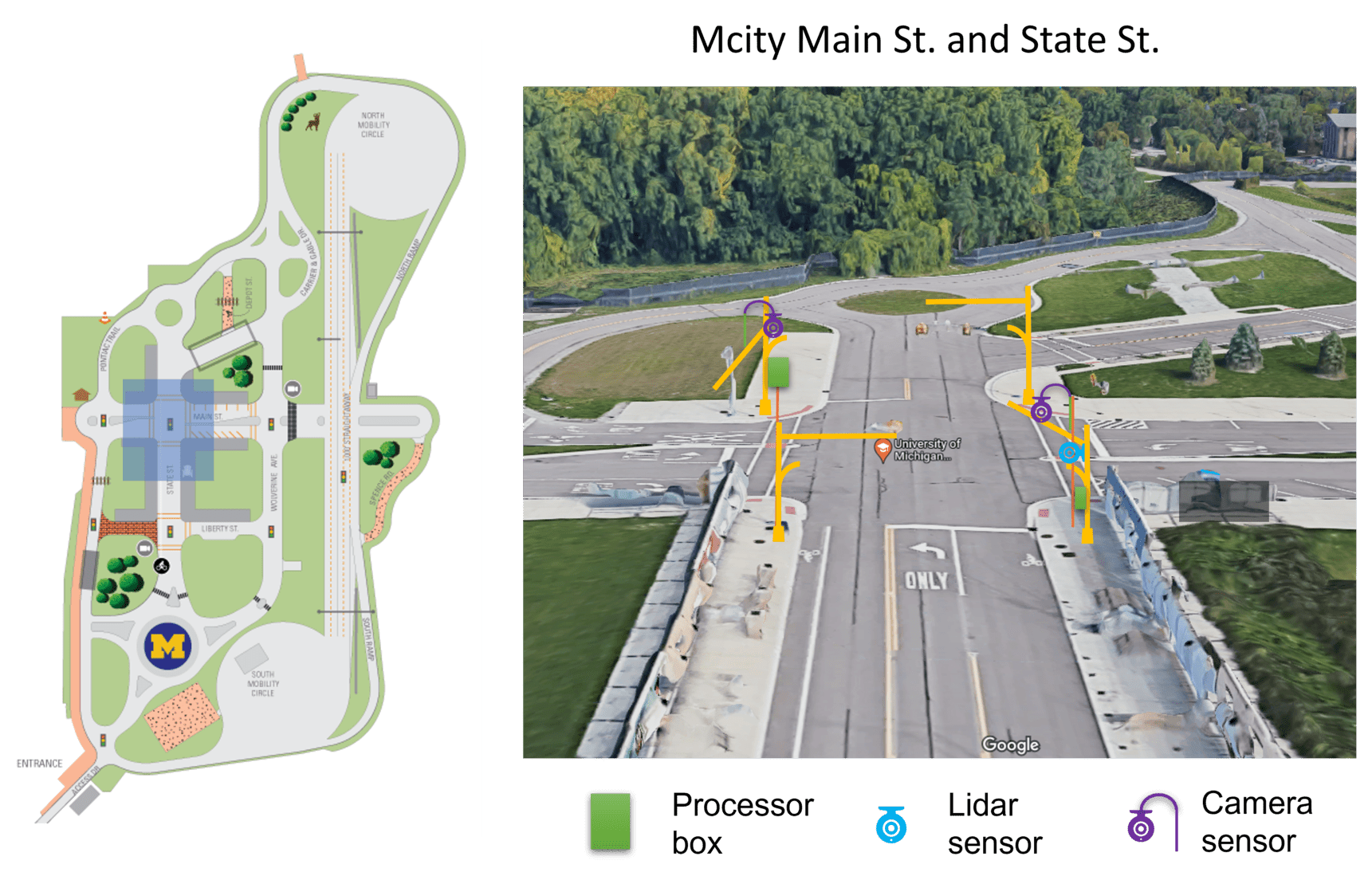} 
        \caption{Sensor placement and a view of Mcity}
        \label{fig:mcity}
    \end{subfigure}%
    \begin{subfigure}{.5\textwidth}
        \centering
        \includegraphics[width=.9\linewidth]{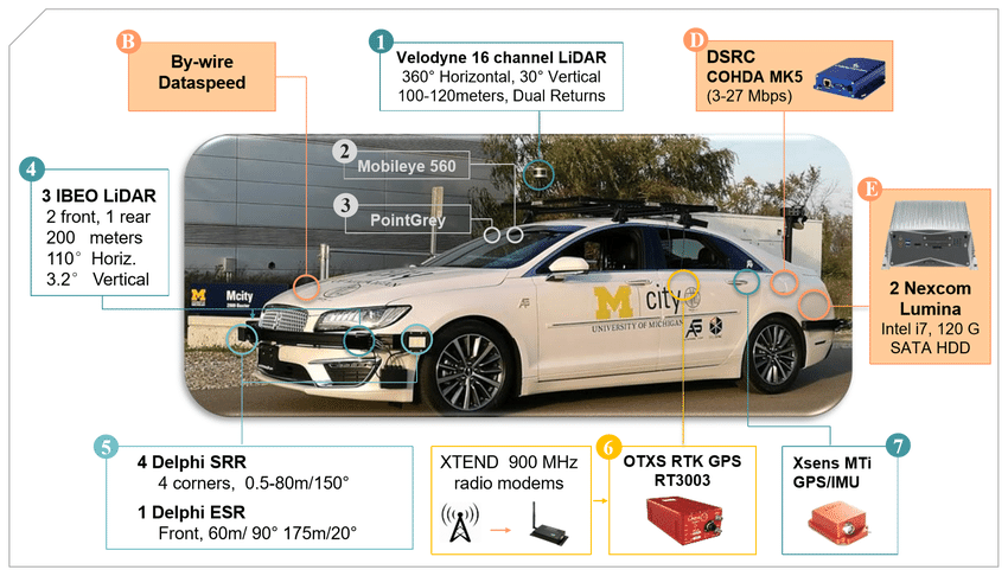} 
        \caption{Depiction of the experimental vehicle setup}
        \label{fig:experiment-vehicle}
    \end{subfigure}
    \caption{Illustrations of the experimental setup for roadside perception systems testing in Mcity.}
    \label{fig:main}
\end{figure*}

Given the absence of direct access to the systems under evaluation and the 'black box' nature of their perception algorithms, the evaluation focuses solely on the final output of the systems. 
We operate under the assumption that the perception system will periodically output a list of detected entities. Each entry in this list should contain a minimum of four attributes: latitude, longitude, category (distinguishing between vehicle or pedestrian), and a unique identifier (id). During each experimental trial, we collect two sets of data: the detected trajectories from the perception system and the corresponding ground truth trajectories obtained from the RTK GPS data.
The primary objective and problem being addressed in this paper is the establishment of a standardized, full-stack evaluation solution that hinges on the use of detection data and ground truth data collected as described above. The goal is to systematically  evaluate, fairly compare, and establish benchmarks for roadside perception systems. 

\section{Methodology}
In this section, we will explore three essential parts of our proposed evaluation methodology. The first part centers around measurement techniques. Here, we introduce a method for estimating two key properties: sensor latency and positioning error. These values are convoluted within the raw experimental data and can't be separately measured, necessitating the development of an estimation method. 
The second element focuses on our selection of metrics. During this portion, we introduce the specific metrics that have been chosen to evaluate the performance of roadside perception systems. The final part pertains to the design of our experimental trials, specifying the planned movement patterns of vehicles and pedestrians during each trial. 

\subsection{Measurement Techniques}

Our evaluation method is solely executed based on the detection data and ground truth data gathered as described in the Section \ref{sec:problem-formulation}. The latency and positioning error values are intrinsically interwoven within the data, making their direct measurement unfeasible. 
Consequently, in this subsection, we focus on our developed estimation method specifically designed to assess these two critical values.

\subsubsection{Assumptions and Mathematical Modeling}
Considering a scenario where a single vehicle is traversing along a one-dimensional straight line, the vehicle's location at a specific time $t$ is denoted as $G(t)$. Concurrently, we denote the detected location at  time $t$ as $D(t)$. Given a specific time, $\mathbf{t_1}$, the vehicle's location is $G(\mathbf{t_1})$. Simultaneously, we can identify a detection point generated by the detection system at the same location, timestamped as $\mathbf{t_2}$. This association establishes the relationship $G(\mathbf{t_1}) = D(\mathbf{t_2})$. We denote the latency at this point as $\mathbf{l}$. The following equation emerges:

\begin{equation}
D(\mathbf{t_2}) = G(\mathbf{t_2}-\mathbf{l}) + e_1 + \mathbf{e_2} = G(\mathbf{t_1})
\label{eq:1}
\end{equation}

In this equation, $e_1$ represents the detection system's constant offset, and $\mathbf{e_2}$ represents the random error, which is a zero-mean random variable. It's crucial to note that Equation (\ref{eq:1}) always holds true, as any error can be deconstructed into a static offset plus a zero-mean random error.

\subsubsection{Latency Measurement}
\label{sss:latency-measurement}
Now, let's consider a scenario where the vehicle is traveling at a constant speed, $v_0$. This results in:

\begin{equation}
G(t) = v_0t
\label{eq:2}
\end{equation}

Substituting Equation (\ref{eq:2}) into Equation (\ref{eq:1}) and denoting $\bm{\tau} = \mathbf{t_2} - \mathbf{t_1}$, we derive:

\begin{equation}
\bm{\tau}= \mathbf{t_2} - \mathbf{t_1}
= \mathbf{l} - \frac{e_1}{v_0} - \frac{\mathbf{e_2}}{v_0}
\label{eq:3}
\end{equation}

It's essential to observe that in Equation (\ref{eq:3}), $\bm{\tau, l, e_2}$ are random variables, while the rest are static terms. If we apply the expectation operation to both sides of Equation (\ref{eq:3}), we get:

\begin{equation}
E[\bm\tau] = E[\bm{l}] - \frac{e_1}{v_0}
\label{eq:4}
\end{equation}

Similarly, when the vehicle is traveling at speed $-v_0$:

\begin{equation}
E[\bm\tau'] = E[\bm{l}] + \frac{e_1}{v_0}
\label{eq:5}
\end{equation}

Adding Equations (\ref{eq:4}) and (\ref{eq:5}) yields:

\begin{equation}
E[\bm\tau] + E[\bm\tau'] = 2E[\bm{l}]
\label{eq:6}
\end{equation}

In Equation (\ref{eq:6}), $\bm\tau$ and $\bm\tau'$ are the two random variables that can be directly observed and sampled. Consequently, the expectation of these two random variables can be estimated, thereby allowing us to estimate the expected value of the latency.

Following this mathematical model, we designed an experiment to sample the variables $\bm\tau$ and $\bm\tau'$, and hence estimate the average latency. As depicted in Figure \ref{fig:latency-measurement}, one vehicle drives back and forth in a straight line during the trial. Each trip consists of three areas: an acceleration area, an area of constant speed, and a deceleration area. We have designed the experiment in such a manner that the area of constant speed is ideally located within the optimal detection range of the sensor under evaluation.
In this area, the driver exerts utmost effort to maintain a constant speed of either $v_0$ or $-v_0$. A series of test points are established within this area. At each point, we record the time difference between the detection timestamp and the ground truth timestamp ($\bm{t_2} - \bm{t_1}$ as in equation \ref{eq:3}). Subsequently, by calculating the average of all the sampled time differences, we use this value as the estimated average latency. It's important to note that this experiment must be conducted in pairs of trips (back and forth) in order to neutralize the static error term $\frac{e_1}{v_0}$ in equation (\ref{eq:4}). In practical implementation, a small speed variation is inevitable, Appendix \ref{appendix:1} gives a brief analysis on the effect of speed variation, and Appendix \ref{appendix:2} gives the analysis on the variance of this method.

\begin{figure*}[ht]
\centering
\includegraphics[width=.9\linewidth]{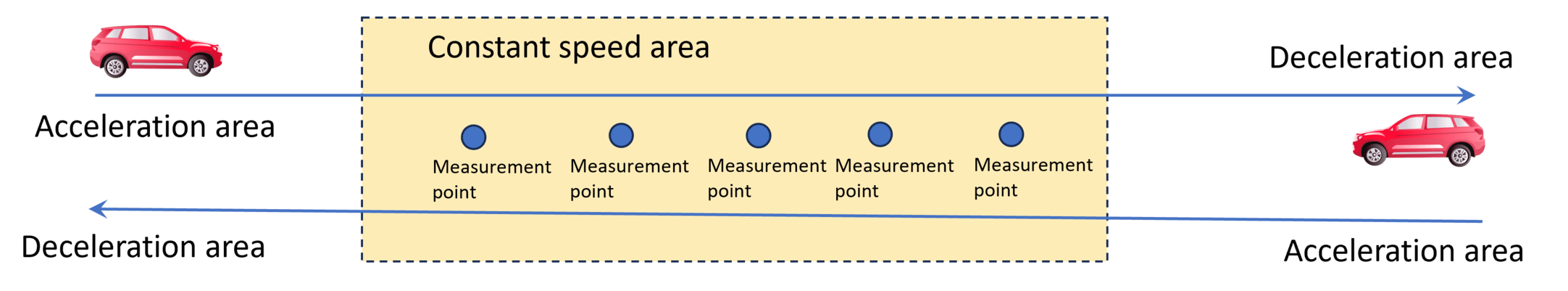}
\caption{Depiction of the experimental setup for latency measurement, showcasing the vehicle's movement path involving acceleration, constant speed, and deceleration areas.}
\label{fig:latency-measurement}
\end{figure*}

\subsubsection{Estimating Positioning Error}
\label{sss:positioning-error}
With the estimated average latency $\bm{\tilde{l}}$ from the method described above, we can compute the estimated positional error for a given detection point using the following estimator: $\bm{\Tilde{e_d}} = D(t) - G(t-\bm{\tilde{l}})$. To demonstrate the viability of this estimator, we present the following theorem, a brief analysis of its variance is also presented in Appendix \ref{appendix:3}.

\begin{theorem}
The estimator $\bm{\Tilde{e_d}}$ is an unbiased estimator of the positional error.
\end{theorem}

\begin{proof}
We start with the expectation of the estimator:
$$E[\bm{\Tilde{e_d}}] = E[D(t) - G(t-E[\bm{l}])]$$
Rearranging and splitting the terms gives us:
$$=E[D(t) - G(t-\bm{l}) + G(t-\bm{l}) - G(t-E[\bm{l}])]$$
Substituting the expression from equation (\ref{eq:1}) into the expectation, we get:
$$E[\bm{\Tilde{e_d}}] = E[e_1 + \bm{e_2}] + v_0E[E[\bm{l}] - \bm{l}] = e_1$$
Thus, we conclude that the estimator is unbiased.
\end{proof}

\subsection{Metrics Selection}
\subsubsection{Data Representation and Preprocessing}
The first step in our methodology requires an understanding of the fundamental concepts of data representation. To begin with, we define \textbf{Data Point} as point representation of a single measured object. At its core, a data point encapsulates at least three attributes: a timestamp, a location (latitude and longitude), and an ID. Data points serve as the most basic and atomic data structure in our evaluation data abstraction. 

Next, we introduce two essential data structures: data frame and trajectory. A Data Frame consists of multiple data points, all sharing an identical timestamp. It encapsulates a singular time frame, encompassing all detected or ground truth objects at that specific time instance. A Trajectory is a collection of data points bearing different timestamps (arranged in ascending order) but sharing the same ID. This structure illustrates the movement path or trajectory of a particular detected or ground truth object over time. The objective of preprocessing is to parse the detection and ground truth data into these structures. We group each set of data by time to create multiple time frames and by ID to form various trajectories. All these data structures are then collectively stored as a '\textbf{trajectory set}'. 

\subsubsection{Matching Mechanisms}
A significant portion of our chosen metrics rely on the correct alignment of data points, frames, and trajectories between the detection results and the ground truth data. For this purpose, we define two types of matching mechanisms: \textbf{Point Matching} and \textbf{Association Matching}.

\paragraph{\textbf{Point Matching}} In this process, we first establish a match between a data frame from the detected data and the ground truth data frame that possesses the closest timestamp (after substracting the estimated latency as described in Section \ref{sss:latency-measurement}). Subsequently, we execute a point-to-point matching between the data points in the detected data frame and the corresponding data points in the ground truth data frame. For this operation, we employ the Hungarian method \cite{kuhn1955hungarian}, which minimizes the matching distance and subsequently achieves an optimal one-to-one correspondence between the points in both frames.

\paragraph{\textbf{Association Matching}:} This matching mechanism amplifies the concept of point matching and applies it to trajectories. Here, we first adopt the Hungarian method to pair detected trajectories with the ground truth trajectories, by maximizing the number of true positives. Following this, we perform \textbf{point matching} within each pair of the matched trajectories.

The two matching mechanisms will yield a set of matched data point pairs from the detected data points to the ground truth data points. It's noteworthy that these matching mechanisms slightly diverge from those utilized in common image-based tracking benchmarks \cite{milan2016mot16,dendorfer2020mot20, dave2020tao}. In our case, we need to incorporate an extra nearest-time matching step. This distinction arises because, unlike standard scenarios where the ground truth is annotated on the identical frames, our ground truth is independently captured by a standalone GPS device, which operates asynchronously with the the detection system. 

\subsubsection{True Positives, False Positives, False Negatives, and ID Switches}
We employ standard terminology used in binary classification tasks to identify and analyze the outcomes of our matching process. An illustrative example is provided below (Figure \ref{fig:classification}) for further clarification.

\begin{itemize}
    \item \textbf{True Positives (TP)}: A True Positive instance arises when a detected data point is successfully matched with a ground truth data point, with the distance between the two not exceeding 1.5 meters. This indicates that the detection system has accurately identified an object and correctly estimated its location within an acceptable margin of error.
    
    \item \textbf{False Positives (FP)}: False Positives denote cases where a data point is detected, but no corresponding ground truth data point exists, or the distance between the detected point and the matched ground truth data point exceeds 1.5 meters. These scenarios suggest that the detection system has erroneously detected an object or incorrectly estimated an object's location beyond the acceptable tolerance.
    
    \item \textbf{False Negatives (FN)}: False Negatives occur when a ground truth data point exists, but the detection system fails to identify a corresponding detected data point. This represents a detection failure by the system.
    
    \item \textbf{Identity Switches (ID Switches)}: An ID Switch happens when the identity assigned to a specific trajectory by the detection system gets wrongly associated with another trajectory. In other words, the tracking system loses track of the original object and erroneously associates the track of a different object to it.
\end{itemize}

The matched pairs of points discussed above can be the outcome of either \textbf{Point Matching} or \textbf{Association Matching}. For simplicity, we denote the true positives, false positives, false negatives obtained from \textbf{Point Matching} as TP, FP, FN, and those from \textbf{Association Matching} as TPA, FPA, and FNA. The 1.5m threshold in accordance with the SAE2945 standard \cite{sae2016board}.




\begin{figure*}[ht]
\centering
\begin{subfigure}[b]{0.42\textwidth}
    \centering
    \includegraphics[width=\textwidth]{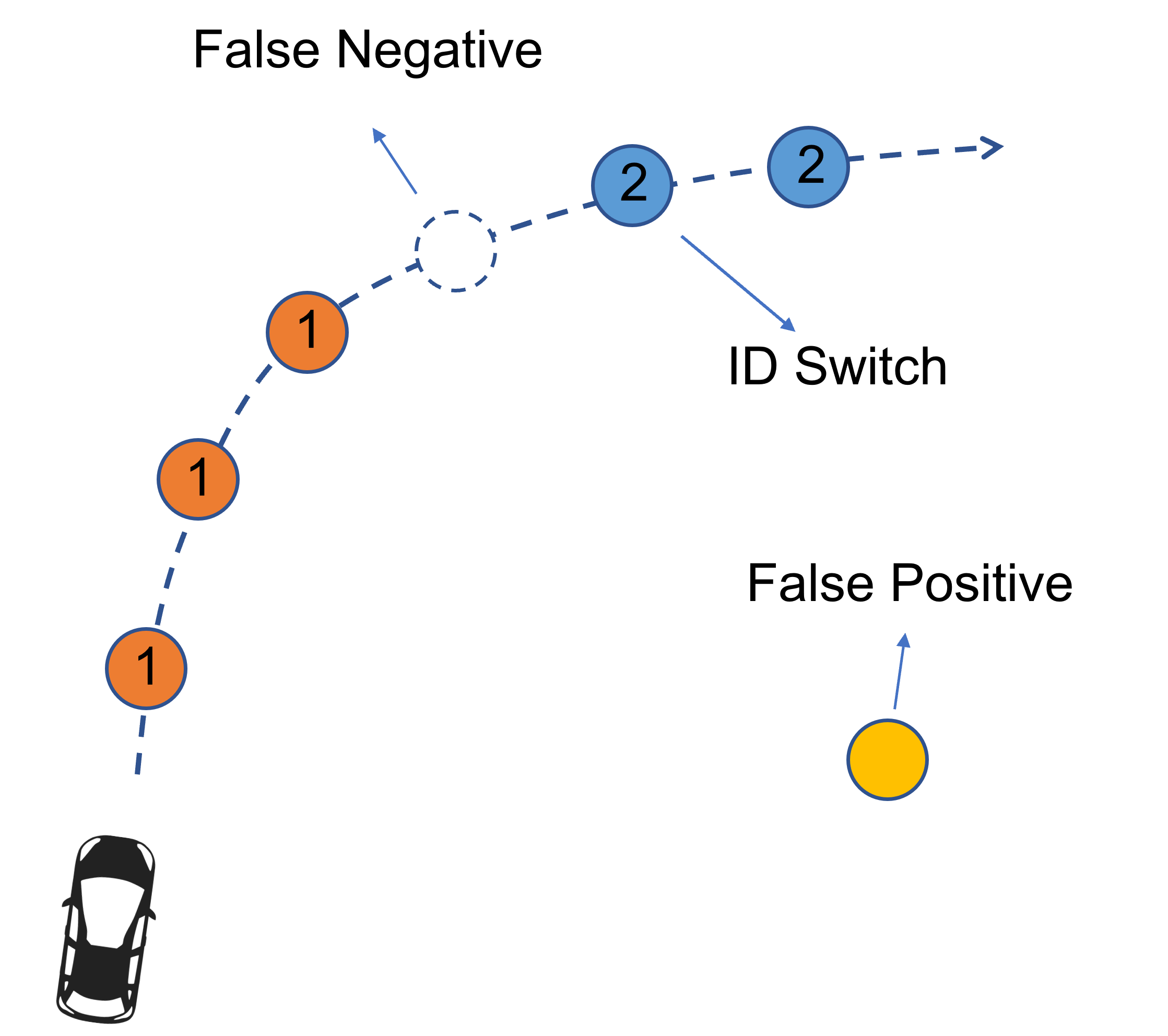}
    \caption{Illustration of True Positives (TP), False Positives (FP), False Negatives (FN), and Identity Switches (ID Switches).}
    \label{fig:point_matching}
\end{subfigure}
\hfill
\begin{subfigure}[b]{0.54\textwidth}
    \centering
    \includegraphics[width=\textwidth]{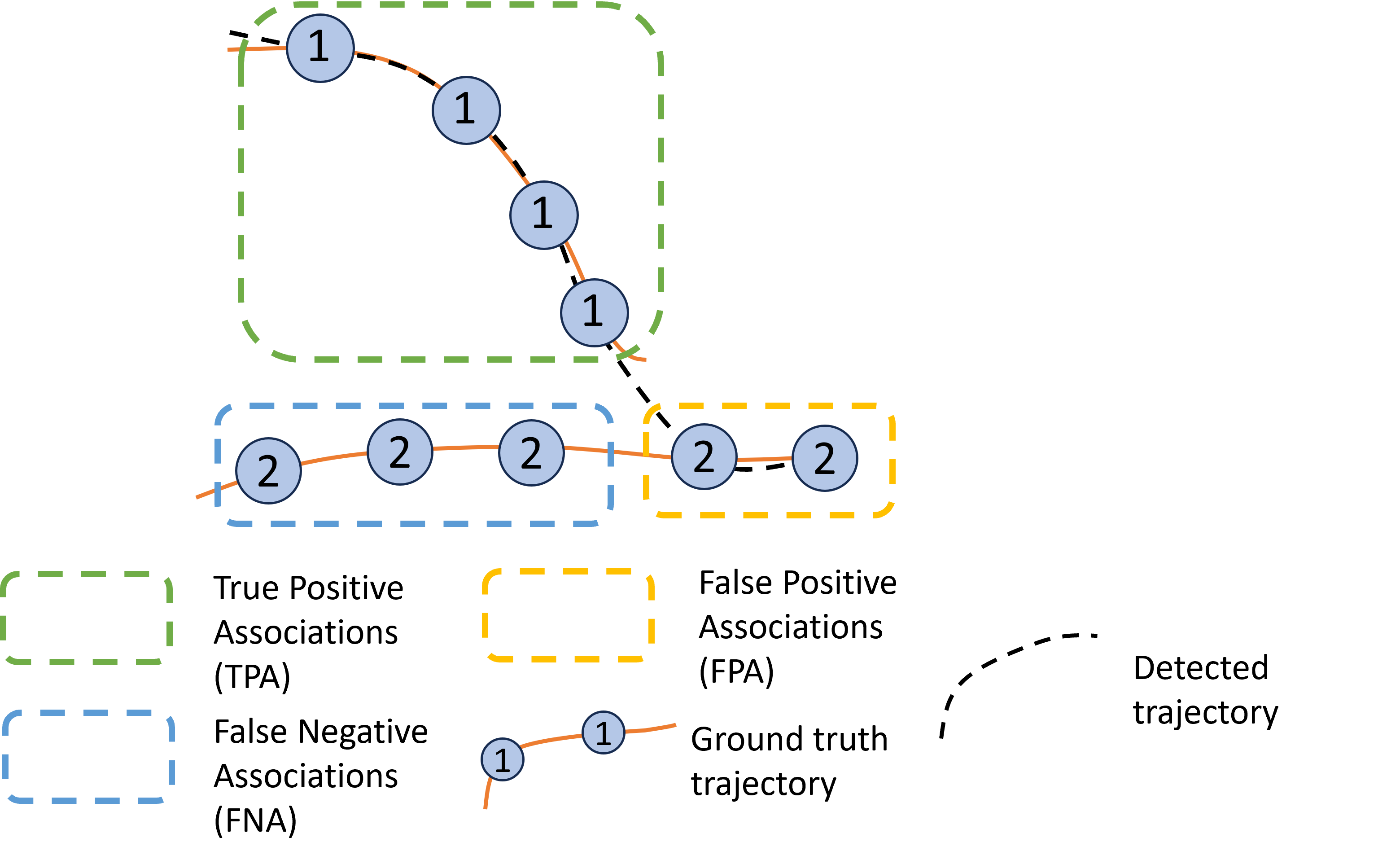}
    \caption{Illustration of True Positives Associations  (TPA), False Positives Associations (FPA), and False Negatives Associations (FNA) in the context of Association Matching.}
    \label{fig:association_matching}
\end{subfigure}
\caption{Illustration of True Positives, False Positives, False Negatives, and Identity Switches in the context of our matching process.}
\label{fig:classification}
\end{figure*}

\subsubsection{Multiple Object Tracking Precision (MOTP)}
Multiple Object Tracking Precision (MOTP) \cite{ristani2016performance} is a measure of the tracking algorithm's precision in localizing the detected objects. It is computed as the total distance between each True Positive (TP) and its corresponding ground truth object, divided by the total number of True Positives. Mathematically, it can be defined as:

\begin{equation}
MOTP = \frac{\sum_{t,i}d_{t,i}}{\sum_t c_t}
\end{equation}

where $d_{t,i}$ is the distance of the $i$th match in frame $t$, and $c_t$ is the number of matches in frame $t$. Lower MOTP values correspond to better localization precision of the detection algorithm.

\subsubsection{Multiple Object Tracking Accuracy (MOTA)}
Multiple Object Tracking Accuracy (MOTA) \cite{ristani2016performance} evaluates the overall tracking performance, taking into account False Positives (FP), False Negatives (FN), and Identity Switches (IDSw). It is calculated as follows:

\begin{equation}
MOTA = 1 - \frac{\sum_t (FP_t + FN_t + IDSw_t)}{\sum_t g_t}
\end{equation}

where $FP_t$ is the number of false positives in frame $t$, $FN_t$ is the number of false negatives in frame $t$, $IDSw_t$ is the number of identity switches in frame $t$, and $g_t$ is the number of ground truth objects in frame $t$. Higher MOTA scores indicate better tracking accuracy. Notice the FP, FN are calculated with \textbf{point matching}.

\subsubsection{IDF1 Score}

The IDF1 score \cite{ristani2016performance} is a measure that encapsulates both the precision and recall of the identification process in multi-object tracking. It is particularly designed for tasks where maintaining consistent identities of the tracked objects is of significant importance.

To elaborate further, the IDF1 score is the harmonic mean of Identification Precision (IDP) and Identification Recall (IDR):

\begin{itemize}
    \item \textbf{Identification Precision (IDP)}: IDP is the proportion of correctly identified detections out of all detections. Mathematically, it's defined as $IDP = \frac{TPA}{TPA+FPA}$, where $TPA$ is the number of true positive associations and $FPA$ is the number of false positive associations.
    
    \item \textbf{Identification Recall (IDR)}: IDR, on the other hand, is the ratio of correctly identified detections to the total number of ground truth objects. It's defined as $IDR = \frac{TPA}{TPA+FNA}$, with $FNA$ representing the number of false negative associations.
\end{itemize}

Combining these two measures, the IDF1 score is computed as:

\begin{equation}
IDF1 = \frac{2 \times IDP \times IDR}{IDP + IDR} = \frac{2 \times TPA}{2 \times TPA + FPA + FNA}
\end{equation}

In essence, a high IDF1 score indicates that not only the detection system correctly identified and located the objects, but also consistently maintained their identities throughout the tracking process.

\subsubsection{Higher Order Tracking Accuracy (HOTA)}

Higher Order Tracking Accuracy (HOTA) \cite{luiten2021hota} is a comprehensive metric that accounts for both detection and association accuracy in object tracking scenarios. HOTA is based on two core components: Detection Accuracy (DetA) and Association Accuracy (AssA).

Detection Accuracy (DetA) assesses the ability of the system to correctly detect objects, without considering their identities. The formula to calculate DetA is given by:

\begin{equation}
DetA = \frac{TP}{TP + FP + FN}
\end{equation}

Association Accuracy (AssA), on the other hand, evaluates the capacity of the system to correctly associate detections to the same object over time, in essence, evaluating the tracking aspect. AssA is computed using:

\begin{equation}
AssA = \frac{TPA}{TPA + FPA + FNA}
\end{equation}


The overall HOTA score is then calculated by taking the geometric mean of DetA and AssA, which balances both detection and association aspects.

\subsection{Experiment Trial Design}
The experimental trials are designed to evaluate different metrics across a diverse range of scenarios. For this purpose, we choose the intersection of State Street and Main Street in Mcity as our trial location. This intersection, with two lanes on each street, represents a medium-sized intersection scenario that is commonplace in real-world urban transportation networks. The trials are organized into four distinct categories: latency trials, one-vehicle trials, one-vehicle-with-pedestrian trials, and two-vehicle-with-pedestrian trials. 

The first two trials are designed specifically for latency evaluation. Each of these trials involves a vehicle driving back and forth in a straight line, maintaining a consistent speed within predefined constant speed zones. The first trial is conducted along the east-west direction, while the second trial follows the north-south direction. These trials are designed according to section \ref{sss:latency-measurement} and shown in Fig. \ref{fig:latency_trial}. The latency for each trial is estimated independently, and the final latency value is calculated as the average of the two trials. This estimated latency is then utilized in the evaluation of subsequent trials.

The second set of trials comprises eight individual trials, each involving a single vehicle. These trials cover all possible maneuvers at the intersection, providing a comprehensive examination of the detection system's performance in diverse vehicle movement scenarios as shown in Fig. \ref{fig:one_vehicle_trial}.
This collection of trials aims to provide direct observation on the system's geographical coverage and its performance in handling diverse movement scenarios. Moreover, these trials offer detailed, granular insights into the system's performance by gauging its specific responsiveness to individual vehicle movements. By examining the system's performance under these conditions, we can gain a nuanced understanding of its capabilities and potential areas for improvement.

The subsequent pair of trials incorporates both a vehicle and a pedestrian. One trial directs the vehicle along an east-west axis, while the alternate trial positions the vehicle on a north-south trajectory. These trials, as depicted in Fig. \ref{fig:one_vehicle_pedestrian_trial}, have been crafted to examine the detection system's performance under a mixed scenario, where both vehicles and pedestrians are in play.
The system's performance under these scenarios directly informs us of its capabilities in ensuring the safety of vulnerable road users (VRUs) when implemented in a production setting. 

The final set of trials involves three tests, each involving two vehicles and a pedestrian. Two trials simulate car-following situations (one in the east-west direction, as shown in Fig. \ref{fig:two_vehicle_pedestrian_trial_1} and the other in the north-south direction, as shown in Fig. \ref{fig:two_vehicle_pedestrian_trial_2}), and the third involves two vehicles approaching the intersection from perpendicular directions, as illustrated in Fig. \ref{fig:two_vehicle_pedestrian_trial_3}. 
These trials are designed to provide an evaluation of the system under complex scenarios. Factors such as occlusion of the following vehicle by the leading vehicle in car-following situations and the need to distinguish multiple detected objects pose additional challenges. Particularly, the trial involving perpendicularly approaching vehicles tests the system's capacity to warn of intersection crashes and red-light violations. 





\begin{figure*}[ht]
\centering
\begin{subfigure}{.32\textwidth}
\centering
\includegraphics[width=.99\textwidth]{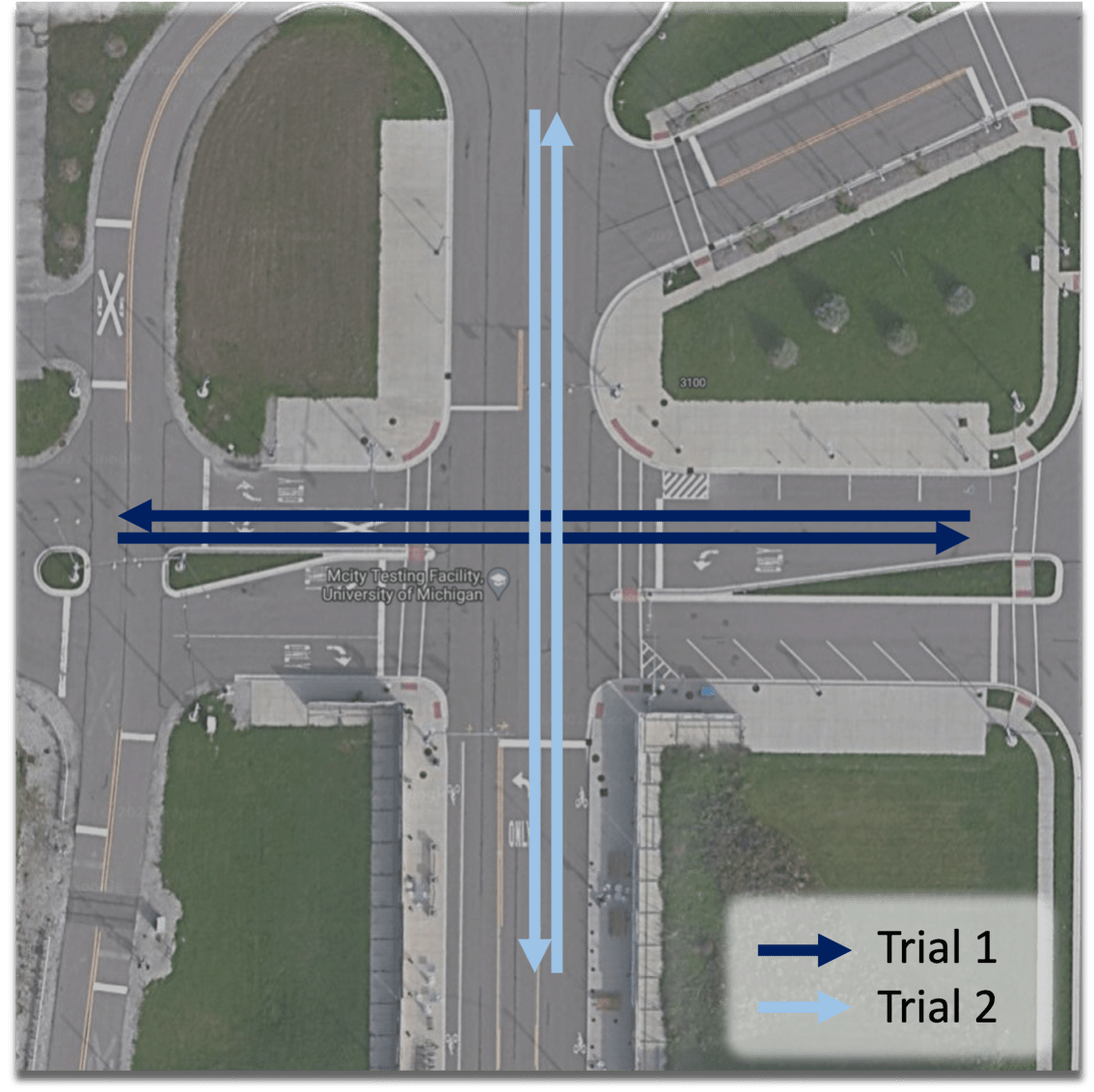}
\caption{Latency Trials}
\label{fig:latency_trial}
\end{subfigure}
\begin{subfigure}{.32\textwidth}
\centering
\includegraphics[width=.99\textwidth]{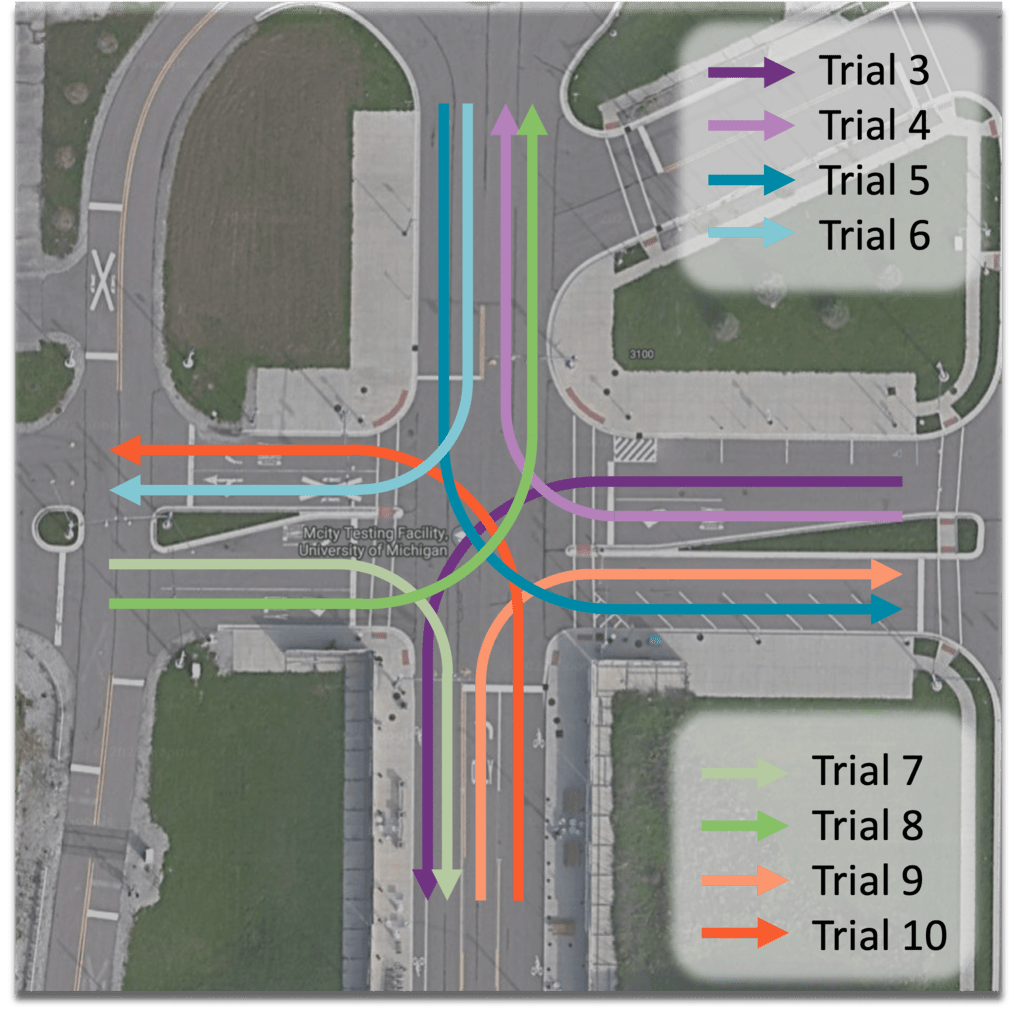}
\caption{One Veh. Trials}
\label{fig:one_vehicle_trial}
\end{subfigure}
\begin{subfigure}{.32\textwidth}
\centering
\includegraphics[width=.99\textwidth]{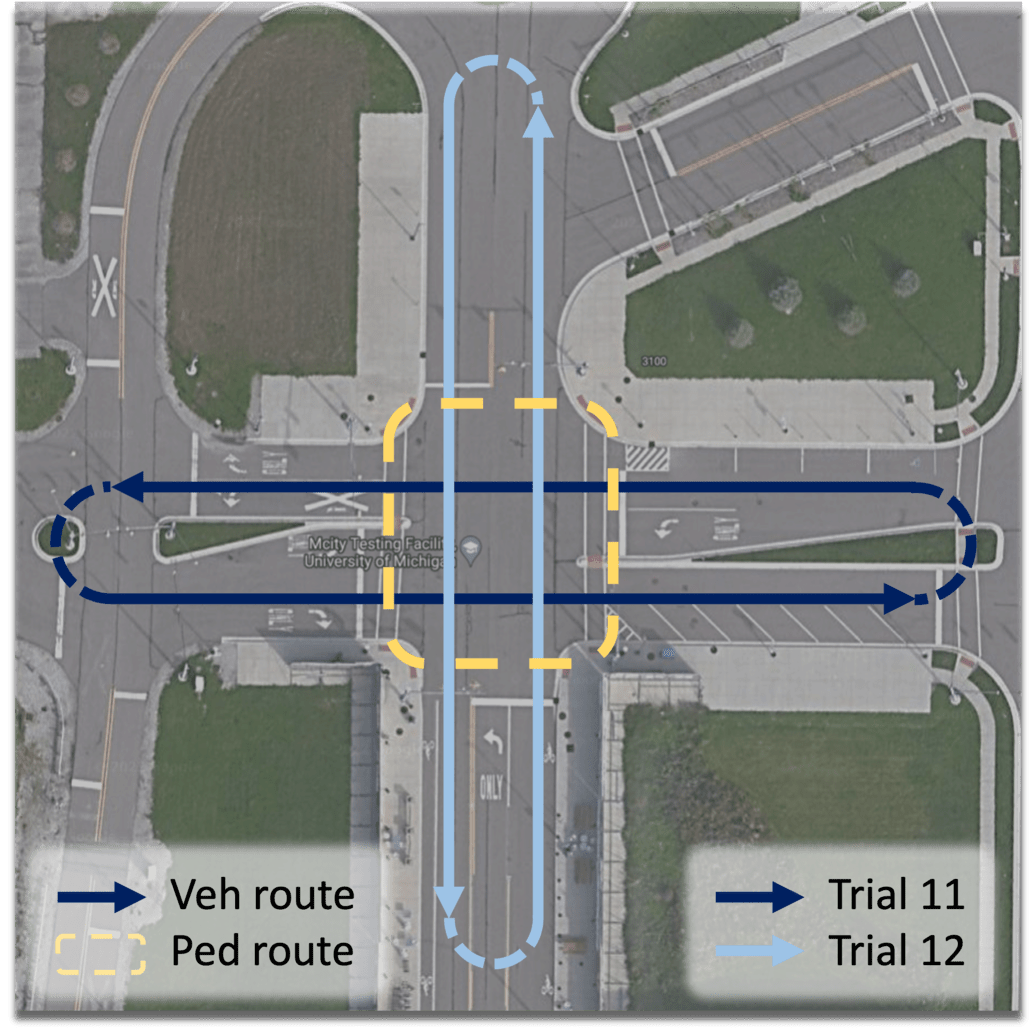}
\caption{One Veh. with Ped. Trials}
\label{fig:one_vehicle_pedestrian_trial}
\end{subfigure}
\begin{subfigure}{.32\textwidth}
\centering
\includegraphics[width=.99\textwidth]{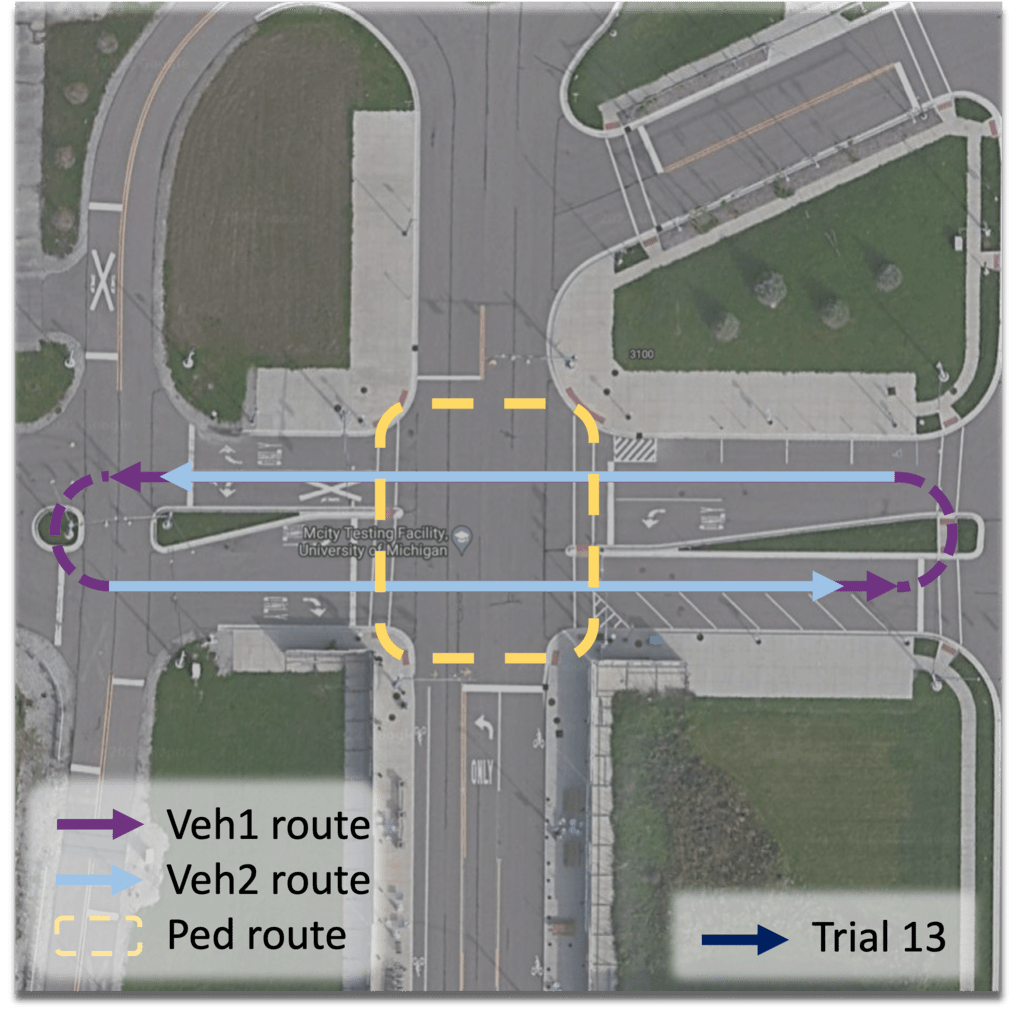}
\caption{Two Veh. with Ped. Trial - 1}
\label{fig:two_vehicle_pedestrian_trial_1}
\end{subfigure}
\begin{subfigure}{.32\textwidth}
\centering
\includegraphics[width=.99\textwidth]{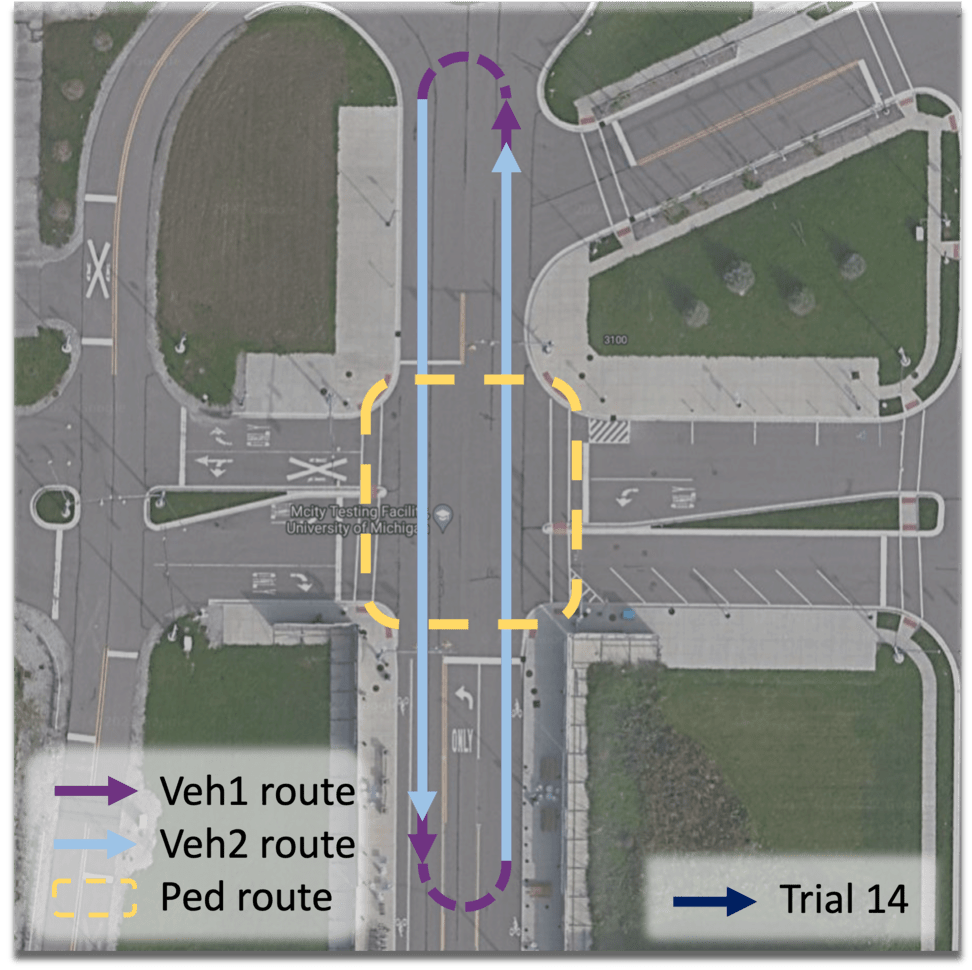}
\caption{Two Veh. with Ped. Trial - 2}
\label{fig:two_vehicle_pedestrian_trial_2}
\end{subfigure}
\begin{subfigure}{.32\textwidth}
\centering
\includegraphics[width=.99\textwidth]{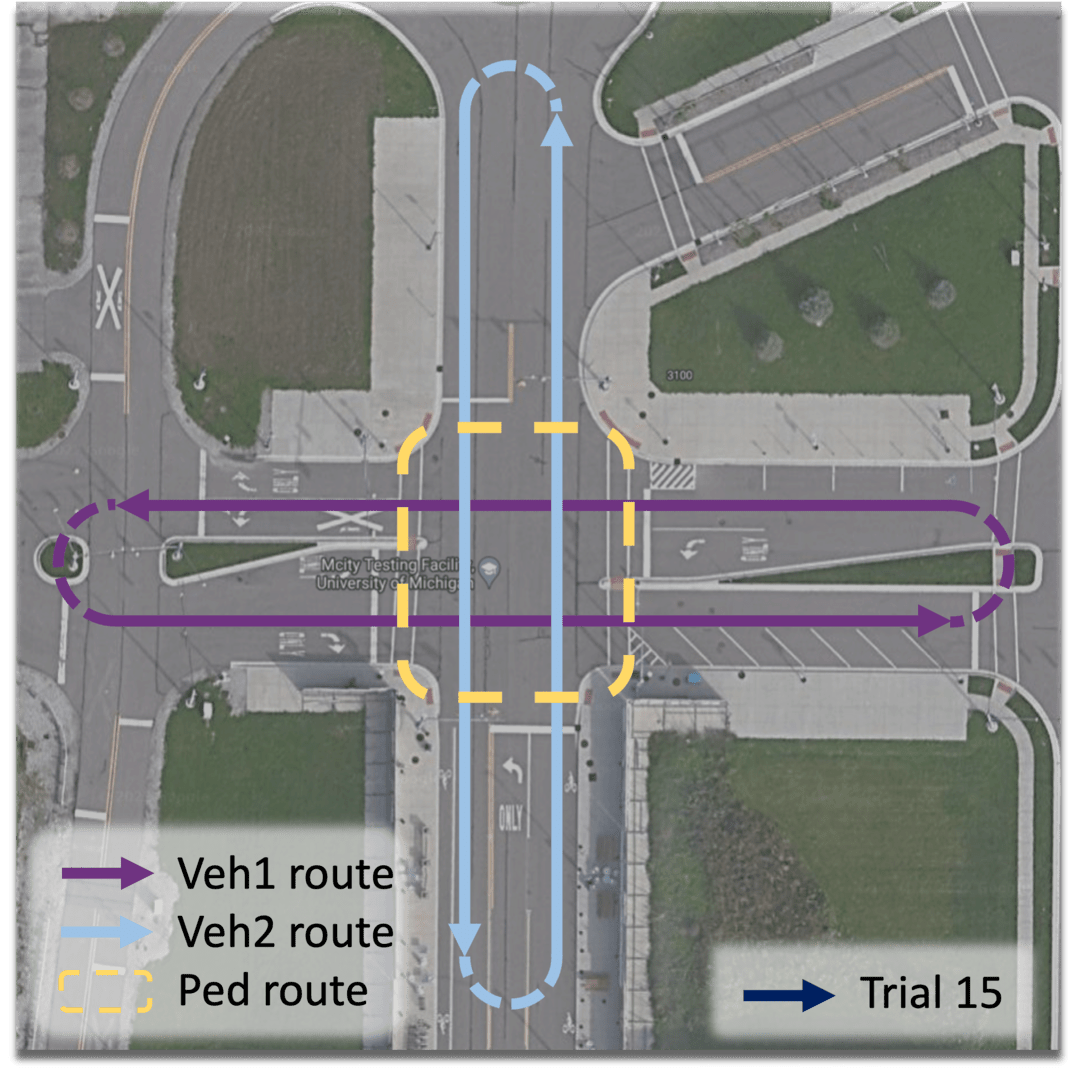}
\caption{Two Veh. with Ped. Trial - 3}
\label{fig:two_vehicle_pedestrian_trial_3}
\end{subfigure}
\caption{Illustrations of all trials designed for evaluation. (a) Latency trials, (b) One Vehicle Trials, (c) One Vehicle with Pedestrian Trials, and (d), (e), (f) Two Vehicles with Pedestrian Trials.}
\label{fig:trial_types}
\end{figure*}

\section{Experiments and Results}

\begin{figure*}[htbp]
    \centering
    \quad\quad Day - Vehicle \quad\quad\quad\quad\quad\quad Night - Vehicle \quad\quad\quad\quad\quad\quad Day - Pedestrian \quad\quad\quad\quad\quad Day - Two Vehicle
    \quad\quad \\
    \rotatebox{90}{\hspace{17mm} System A}
    \begin{subfigure}{.24\textwidth}
        \centering
        \includegraphics[height=\linewidth,width=\linewidth]{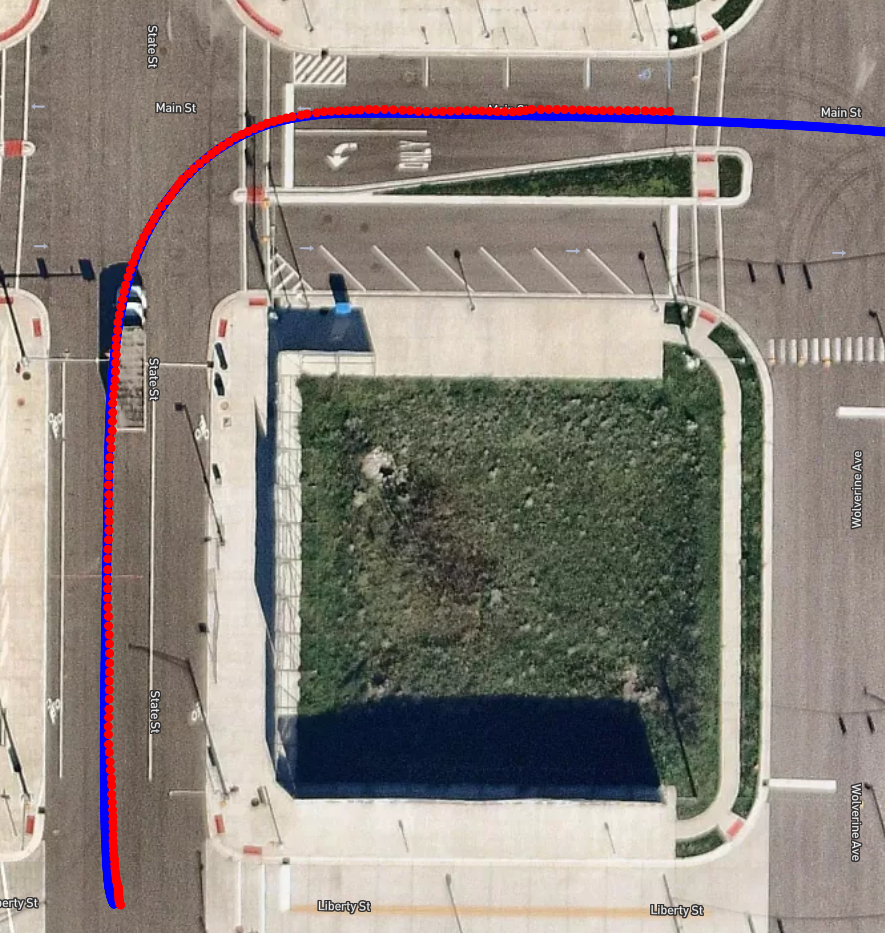}
        \label{fig:system_a_vehicle_day}
    \end{subfigure}\hfill
    \begin{subfigure}{.24\textwidth}
        \centering
        \includegraphics[height=\linewidth,width=\linewidth]{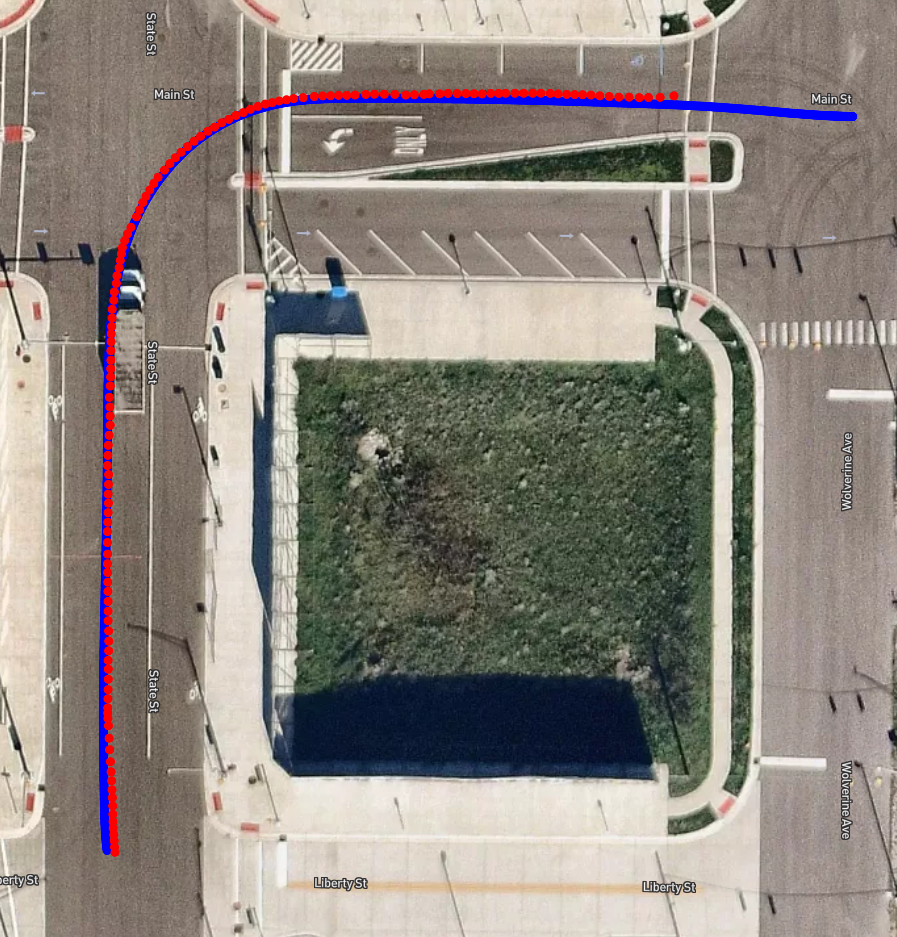}
        \label{fig:system_a_vehicle_night}
    \end{subfigure}\hfill
    \begin{subfigure}{.24\textwidth}
        \centering
        \includegraphics[height=\linewidth,width=\linewidth]{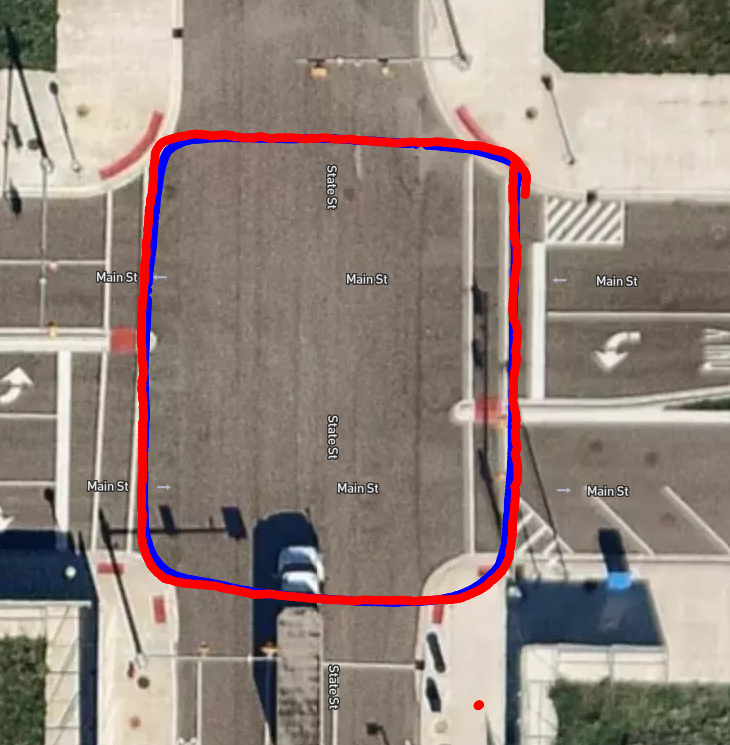}
        \label{fig:system_a_pedestrian}
    \end{subfigure}\hfill
    \begin{subfigure}{.24\textwidth}
        \centering
        \includegraphics[height=\linewidth,width=\linewidth]{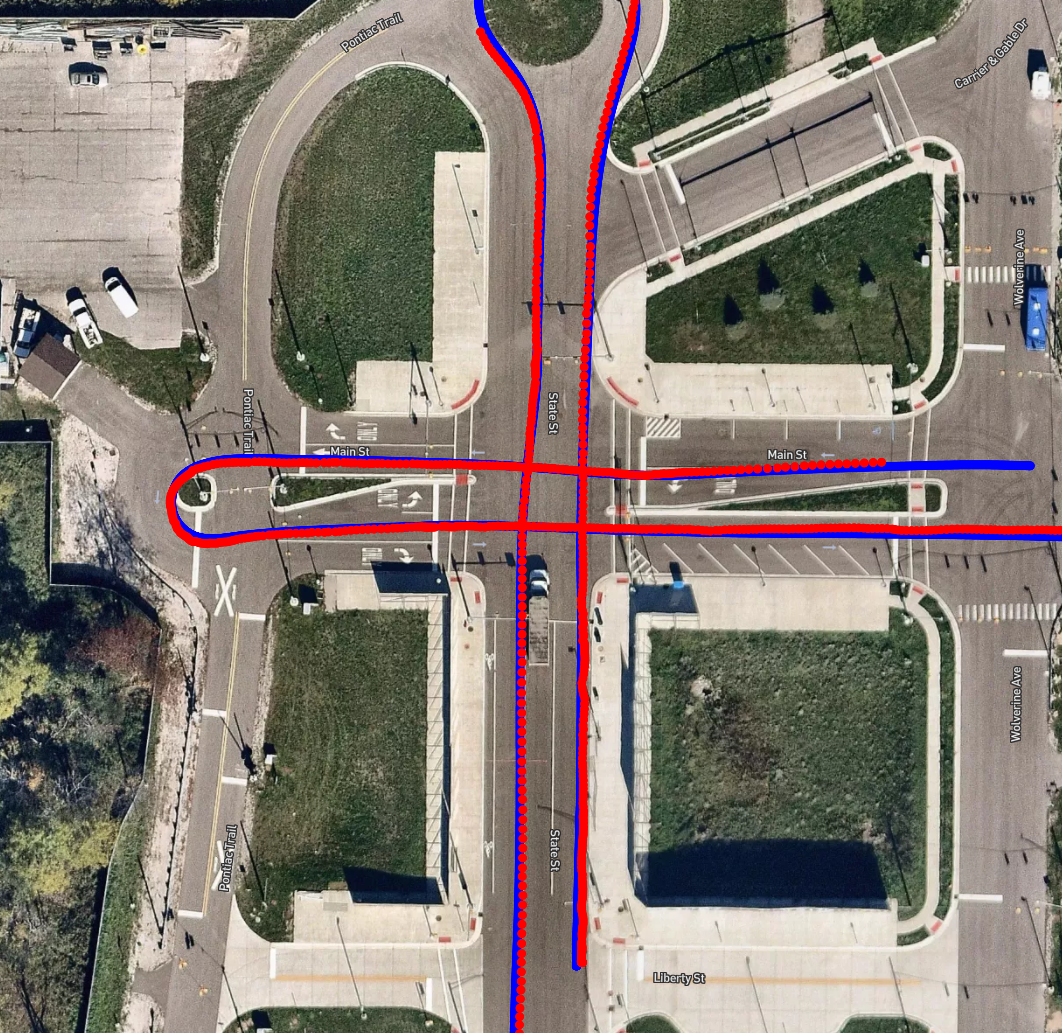}
        \label{fig:system_a_two_vehicles}
    \end{subfigure}%
    \\
    \vspace{-3.2mm}
    \rotatebox{90}{\hspace{17mm} System B}
    \begin{subfigure}{.24\textwidth}
        \centering
        \includegraphics[height=\linewidth,width=\linewidth]{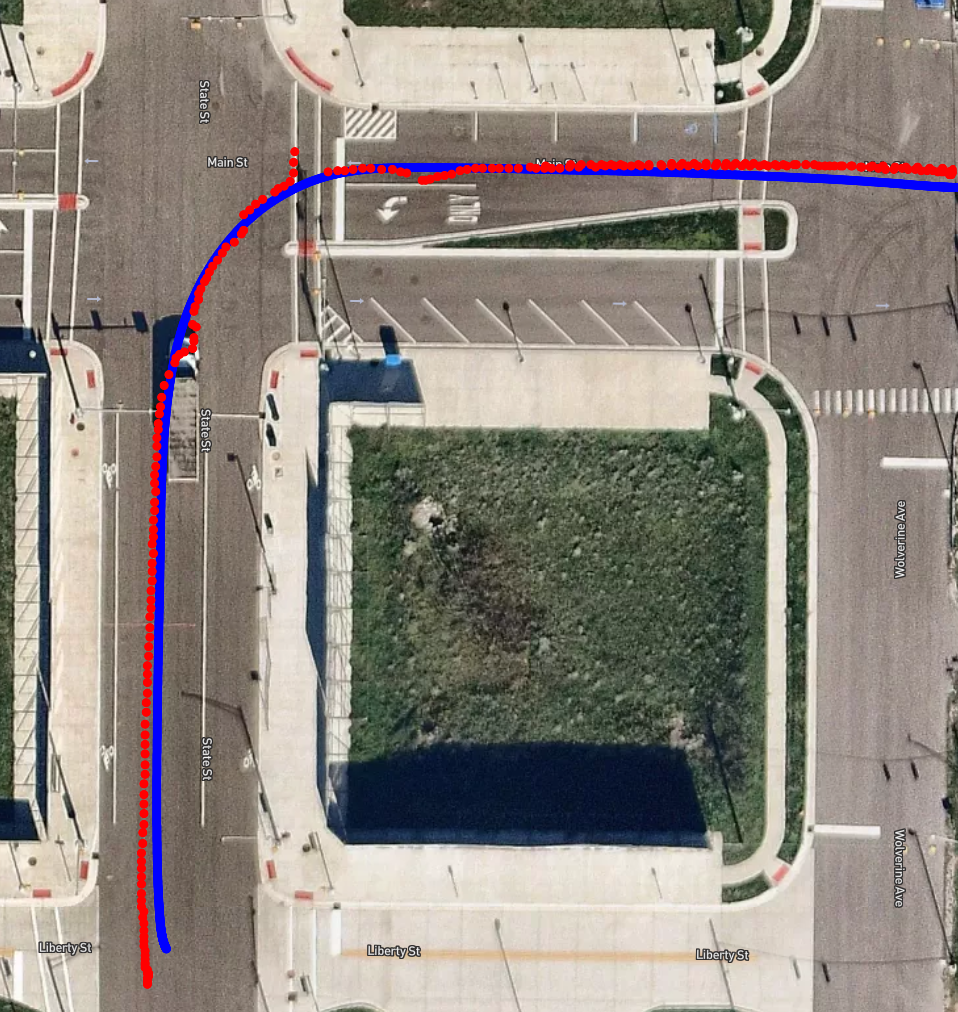}
        \label{fig:system_b_vehicle_day}
    \end{subfigure}\hfill
    \begin{subfigure}{.24\textwidth}
        \centering
        \includegraphics[height=\linewidth,width=\linewidth]{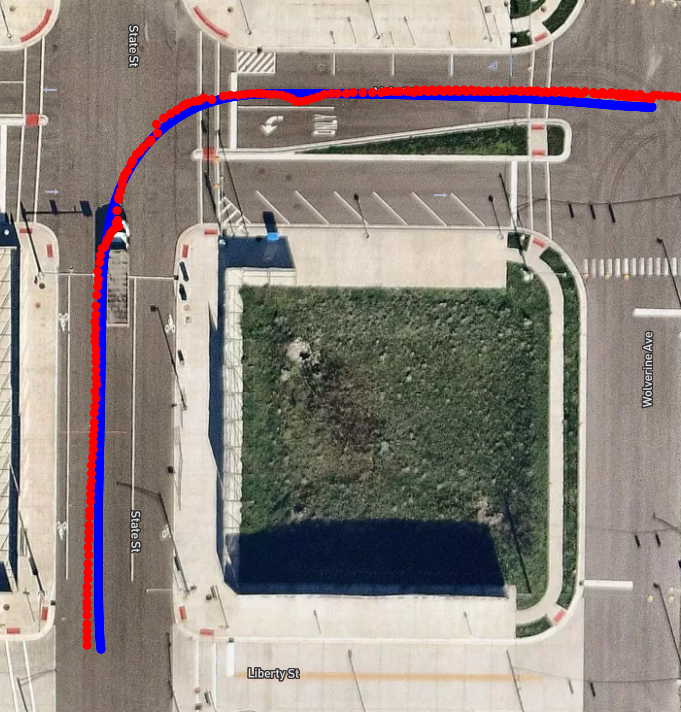}
        \label{fig:system_b_vehicle_night}
    \end{subfigure}\hfill
    \begin{subfigure}{.24\textwidth}
        \centering
        \includegraphics[height=\linewidth,width=\linewidth]{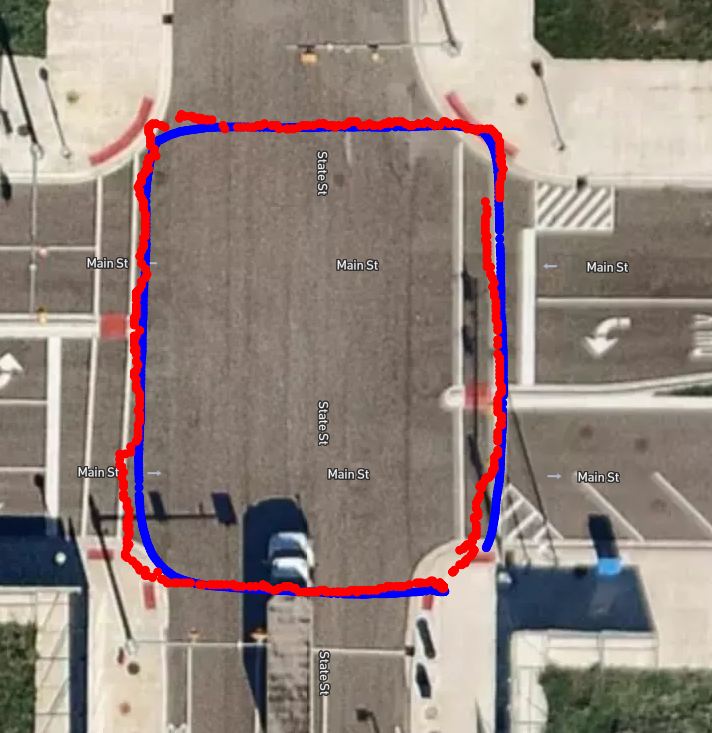}
        \label{fig:system_b_pedestrian}
    \end{subfigure}\hfill
    \begin{subfigure}{.24\textwidth}
        \centering
        \includegraphics[height=\linewidth,width=\linewidth]{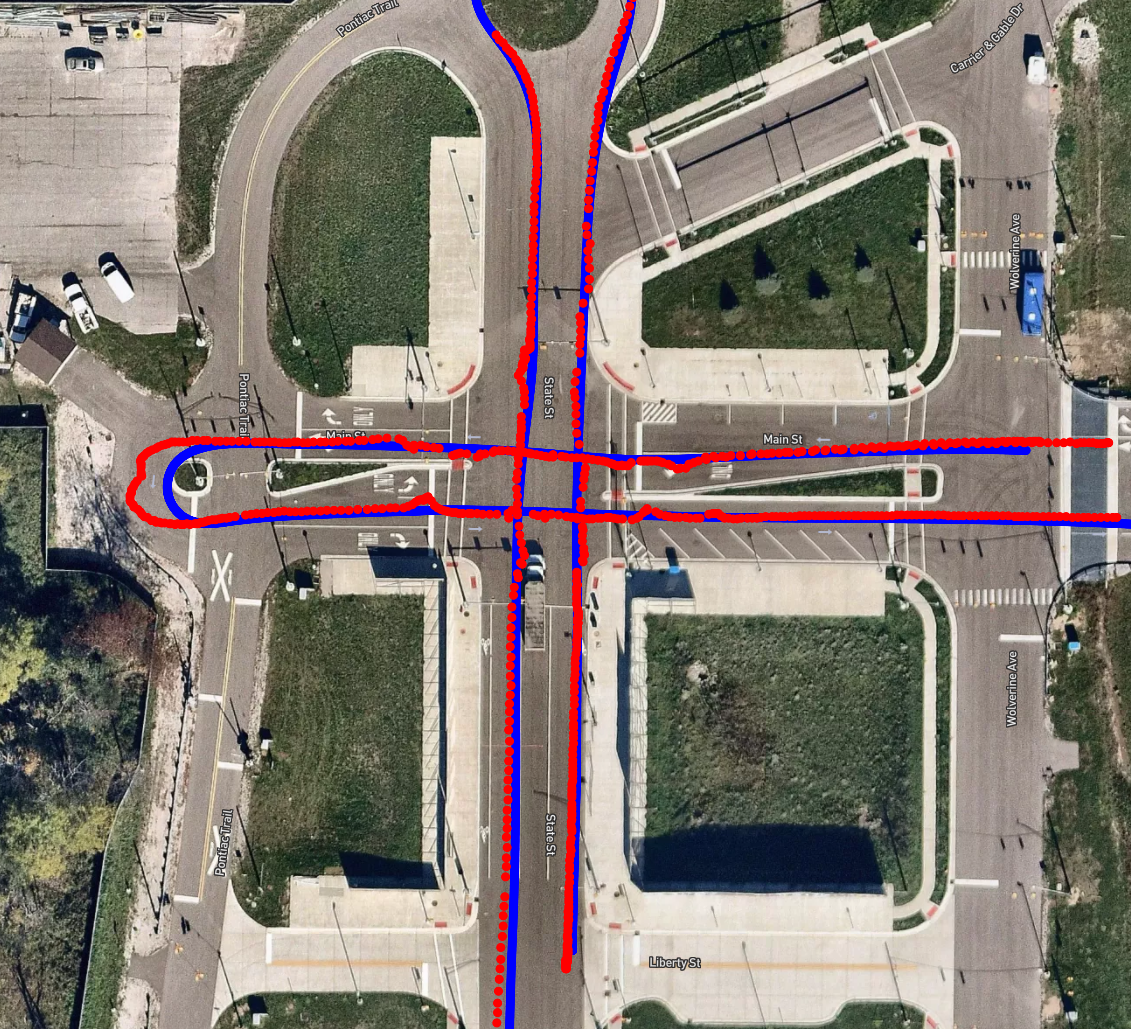}
        \label{fig:system:b_two_vehicles}
    \end{subfigure}
    \\
    \vspace{-3.2mm}
    \rotatebox{90}{\hspace{17mm} System C}
    \begin{subfigure}{.24\textwidth}
        \centering
        \includegraphics[height=\linewidth,width=\linewidth]{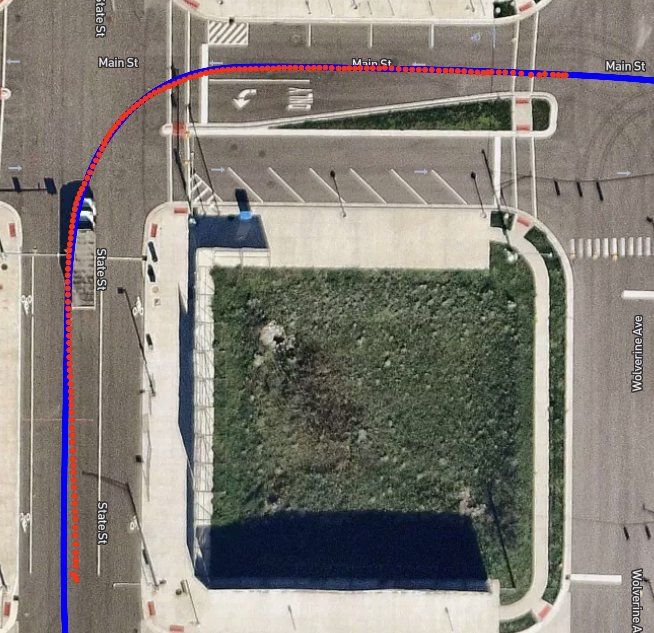}
        \label{fig:system_c_vehicle_day}
    \end{subfigure}\hfill
    \begin{subfigure}{.24\textwidth}
        \centering
        \includegraphics[height=\linewidth,width=\linewidth]{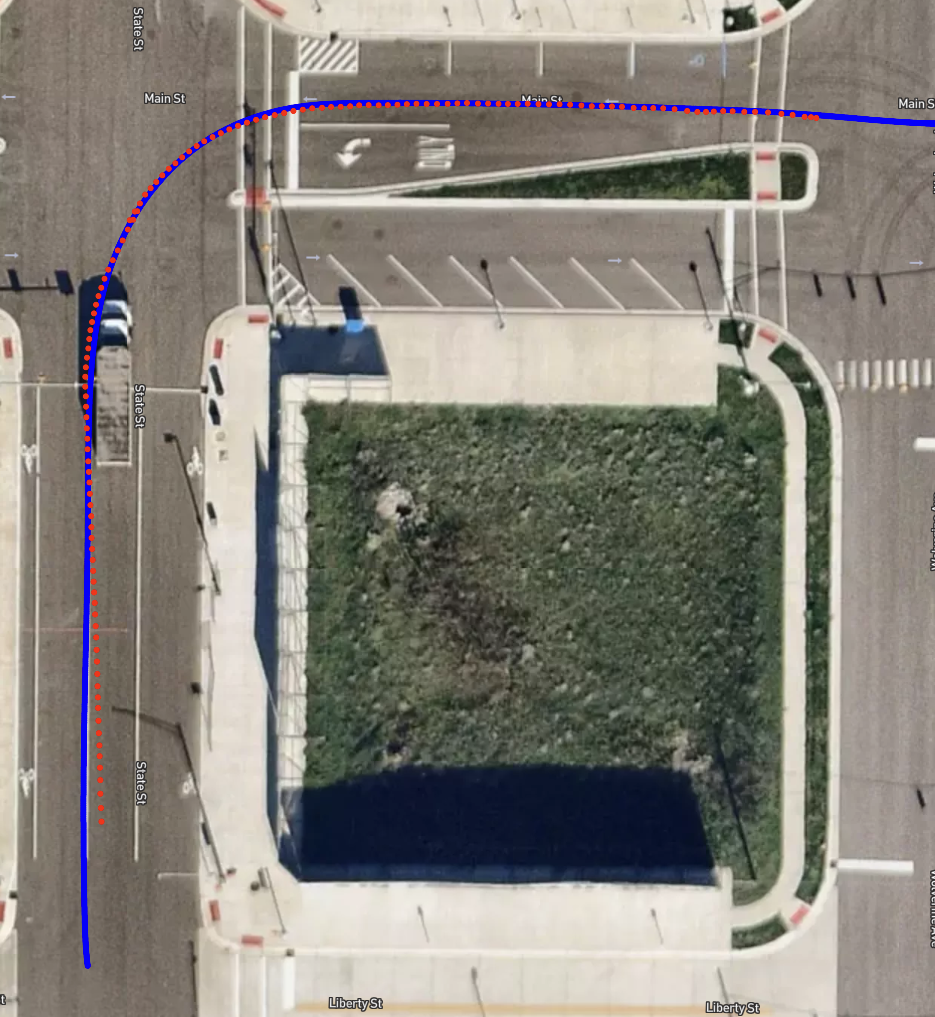}
        \label{fig:system_c_vehicle_night}
    \end{subfigure}\hfill
    \begin{subfigure}{.24\textwidth}
        \centering
        \includegraphics[height=\linewidth,width=\linewidth]{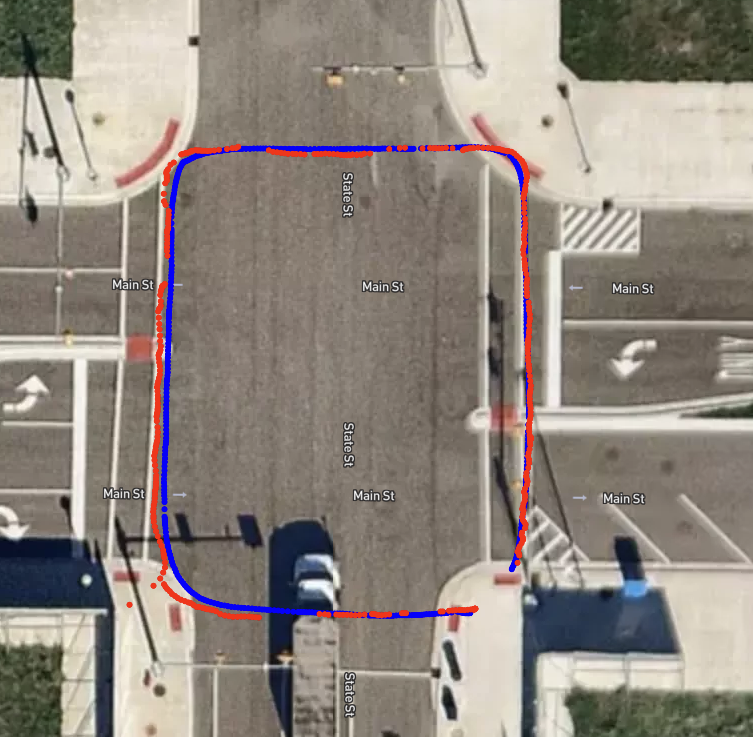}
        \label{fig:system_c_pedestrian}
    \end{subfigure}\hfill
    \begin{subfigure}{.24\textwidth}
        \centering
        \includegraphics[height=\linewidth,width=\linewidth]{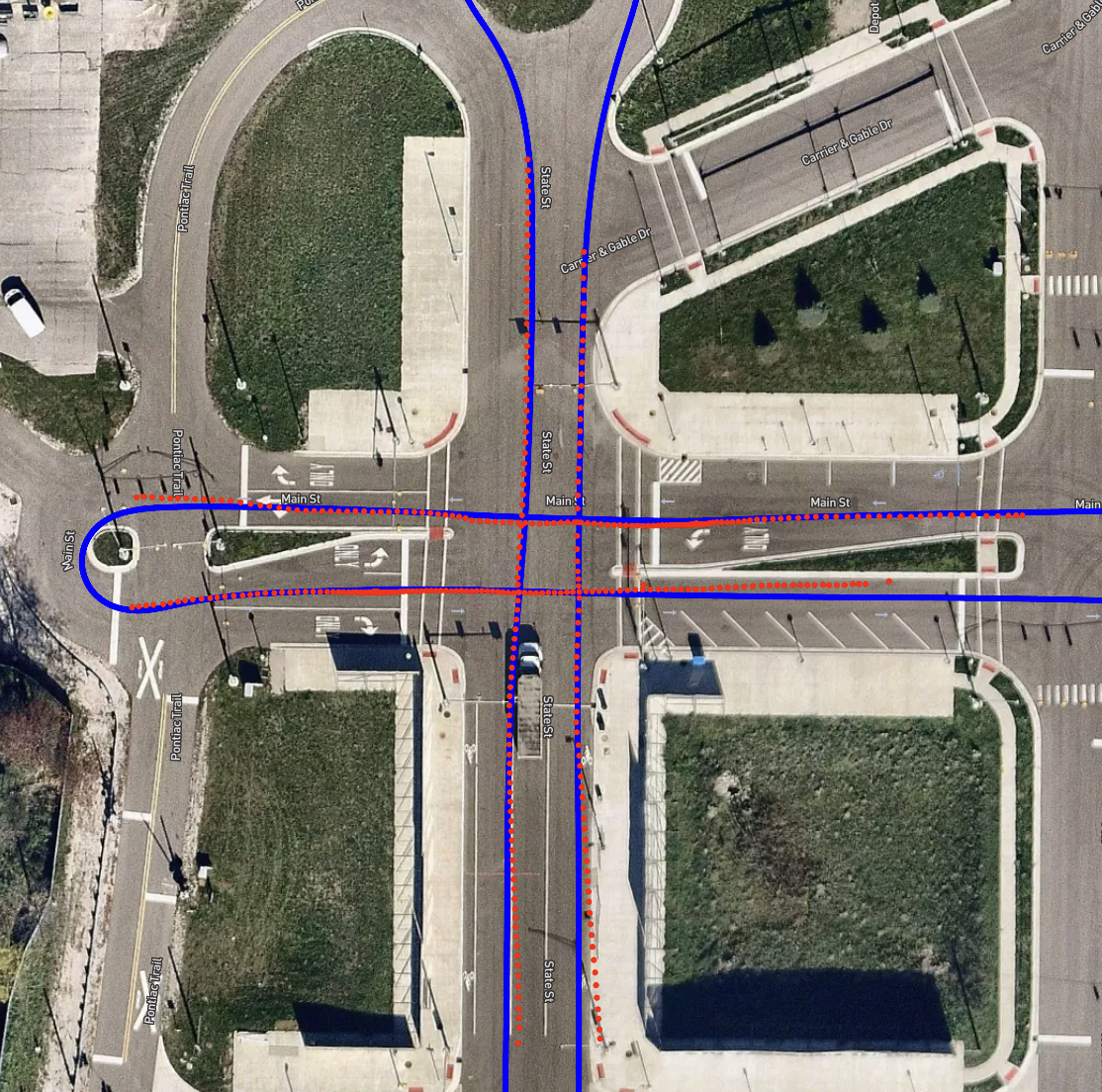}
        \label{fig:system:c_two_vehicles}
    \end{subfigure}
    \caption{The detection results of the  System A (first row), the System B (second row), and the System C (third row) in different trials. The red dots are detection results and the blue dots are ground truth.}
    \label{fig:detection_results}
\end{figure*}

\subsection{Experiments Setup}

The experiments were conducted on October 13th, 2023, in Mcity. Two autonomous vehicles and one pedestrian, each equipped with RTK GPS, were utilized for the trials to provide high-precision ground truth data. The same experiments were performed to evaluate the Three different roadside perception systems. 
To respect the confidentiality agreements in place, we will not be revealing the names of these roadside perception systems under evaluation in our study. Instead, we will refer to them as \textbf{System A}, \textbf{Syestem B} and \textbf{System C} throughout the rest of the paper. Please note that this does not in any way affect the validity or integrity of our evaluation process and results. The \textbf{System A} employs Lidar technology, while \textbf{System B} and \textbf{System C} uses fisheye image sensors. These three systems represent the most popular choices in current roadside perception sensor technology. We also offer the source code used in this study for evaluating these trial data for further research facilitation. The code is accessible on GitHub at \cite{perception_evaluation_code}.

We conducted the trials at two different times: during the daytime at 13 PM and in evening at 6 PM. This strategic choice allowed us to evaluate the performance of the perception systems under varying light conditions. We executed the experimental trials, as outlined in the previous section, for each system. The detection results and the ground truth data were diligently recorded for subsequent analysis. The details and results of the experiments are presented in the following section.

\subsection{Results}
\begin{table}[ht]
\centering
    \setlength{\tabcolsep}{8pt}
\caption{Latency Measurements of Perception Systems}
\begin{tabular}{l|cccc}
\shline
System & North-South & East-West & Mean & Std. \\
\shline
System A & $41$ ms & $54$ ms & $48$ ms & $0.025$\\
System B & $137$ ms & $153$ ms & $145$ ms & $0.106$\\
System C & $1740$ ms & $1690$ ms & $1715$ ms & $0.061$\\
\shline
\end{tabular}
\label{tab:latency}
\end{table}

\begin{table*}[ht]
    \renewcommand{\arraystretch}{1.2}
    \caption{Experimental Results for \textbf{System A} in \textbf{Daytime}}
    \small
    \setlength{\tabcolsep}{4pt}
    \centering
\begin{tabular}{c|c|c|c|c|c|c|c|c}
    \shline
    Category & Trial & FP Rate $(\%)\downarrow$ & FN Rate $(\%)\downarrow$ & IDS $\downarrow$ & MOTA $(\%)\uparrow$ & MOTP (m)$\downarrow$ & IDF1 $(\%)\uparrow$ & HOTA $(\%)\uparrow$ \\
    \shline
    \multirow{12}{*}{Vehicle} & Trial 3 & 0.0 & 7.2 & 0 & 92.0 & 0.402 & 96.3 & 96.3 \\
    & Trial 4 & 0.0 & 7.6 & 0 & 91.8 & 0.478 & 96.1 & 96.1 \\
    & Trial 5 & 0.0 & 0.5 & 0 & 99.3 & 0.386 & 99.7 & 99.7 \\
    & Trial 6 & 0.0 & 0.7 & 0 & 99.3 & 0.373 & 99.7 & 99.7 \\
    & Trial 7 & 0.0 & 0.5 & 0 & 99.4 & 0.507 & 99.8 & 99.8 \\
    & Trial 8 & 0.0 & 0.6 & 0 & 99.4 & 0.424 & 99.7 & 99.7 \\
    & Trial 9 & 0.0 & 0.4 & 0 & 99.3 & 0.415 & 99.8 & 99.8 \\
    & Trial 10 & 0.0 & 1.9 & 0 & 97.8 & 0.376 & 99.0 & 99.0 \\
    & Trial 11 & 0.0 & 2.2 & 0 & 97.9 & 0.408 & 98.9 & 98.9 \\
    & Trial 12 & 0.0 & 1.0 & 0 & 98.8 & 0.555 & 99.5 & 99.5 \\
    & Trial 13 & 0.0 & 2.7 & 0 & 97.3 & 0.470 & 98.6 & 97.3 \\
    & Trial 14 & 0.0 & 0.7 & 0 & 99.2 & 0.568 & 99.6 & 96.6 \\
    & Trial 15 & 0.0 & 1.5 & 0 & 98.5 & 0.523 & 99.2 & 96.7 \\
    \hline
    \multirow{4}{*}{Pedestrian} & Trial 11 & 6.5 & 0.0 & 0 & 93.5 & 0.383 & 96.8 & 93.6 \\
    & Trial 12 & 7.8 & 0.0 & 0 & 92.2 & 0.370 & 96.0 & 92.2 \\
    & Trial 13 & 4.8 & 0.0 & 0 & 95.2 & 0.378 & 97.8 & 95.5 \\
    & Trial 14 & 11.1 & 0.0 & 0 & 88.9 & 0.336 & 94.5 & 89.3 \\
    & Trial 15 & 8.6 & 0.0 & 0 & 91.4 & 0.368 & 95.6 & 91.4 \\
    \shline
\end{tabular}
    \label{tab:exp_results_system_a_day}
\end{table*}

\begin{table*}[ht]
    \renewcommand{\arraystretch}{1.2}
    \caption{Experimental Results for \textbf{System A} in \textbf{Night}}
    \small
    \setlength{\tabcolsep}{4pt}
    \centering
    \begin{tabular}{c|c|c|c|c|c|c|c|c}
    \shline
    Category & Trial & FP Rate $(\%)\downarrow$ & FN Rate $(\%)\downarrow$ & IDS $\downarrow$ & MOTA $(\%)\uparrow$ & MOTP (m)$\downarrow$ & IDF1 $(\%)\uparrow$ & HOTA $(\%)\uparrow$ \\
    \shline
    \multirow{12}{*}{Vehicle} & Trial 3 & 0.0 & 7.2 & 0 & 92.1 & 0.342 & 96.3 & 96.3 \\
    & Trial 4 & 0.0 & 9.5 & 0 & 89.5 & 0.335 & 95.0 & 95.1 \\
    & Trial 5 & 0.0 & 0.7 & 0 & 99.3 & 0.413 & 99.6 & 99.6 \\
    & Trial 6 & 0.0 & 0.5 & 0 & 99.1 & 0.367 & 99.7 & 99.7 \\
    & Trial 7 & 0.0 & 0.7 & 0 & 99.3 & 0.343 & 99.6 & 99.6 \\
    & Trial 8 & 0.0 & 0.8 & 0 & 99.2 & 0.247 & 99.6 & 99.6 \\
    & Trial 9 & 0.0 & 0.2 & 0 & 100.0 & 0.409 & 99.9 & 99.9 \\
    & Trial 10 & 0.0 & 0.0 & 0 & 100.0 & 0.339 & 100.0 & 100.0 \\
    & Trial 11 & 0.0 & 2.3 & 0 & 97.7 & 0.413 & 98.8 & 98.8 \\
    & Trial 12 & 0.0 & 0.2 & 0 & 99.7 & 0.354 & 99.9 & 99.9 \\
    & Trial 13 & 0.0 & 2.1 & 0 & 97.9 & 0.390 & 98.9 & 97.8 \\
    & Trial 14 & 0.0 & 0.4 & 0 & 99.5 & 0.390 & 99.8 & 97.4 \\
    & Trial 15 & 0.0 & 1.6 & 0 & 98.3 & 0.400 & 99.2 & 97.2 \\
    \hline
    \multirow{3}{*}{Pedestrian} & Trial 11 & 1.9 & 0.0 & 0 & 98.1 & 0.382 & 99.0 & 98.1 \\
    & Trial 12 & 2.9 & 0.0 & 0 & 97.1 & 0.345 & 98.7 & 97.2 \\
    & Trial 13 & 5.6 & 0.0 & 0 & 94.4 & 0.341 & 97.2 & 94.5 \\
    & Trial 14 & 2.1 & 0.0 & 0 & 97.9 & 0.324 & 99.0 & 98.0 \\
    \shline
\end{tabular}
    \label{tab:exp_results_system_a_night}
\end{table*}

\begin{table*}[ht]
    \renewcommand{\arraystretch}{1.2}
    \caption{Experimental Results for \textbf{System B} in \textbf{Daytime}}
    \small
    \setlength{\tabcolsep}{4pt}
    \centering
    \begin{tabular}{c|c|c|c|c|c|c|c|c}
    \shline
    Category & Trial & FP Rate $(\%)\downarrow$ & FN Rate $(\%)\downarrow$ & IDS $\downarrow$ & MOTA $(\%)\uparrow$ & MOTP (m)$\downarrow$ & IDF1 $(\%)\uparrow$ & HOTA $(\%)\uparrow$ \\
    \shline
    \multirow{12}{*}{Vehicle} & Trial 3 & 37.2 & 33.6 & 1 & 30.1 & 0.758 & 47.4 & 39.7 \\
    & Trial 4 & 13.7 & 22.1 & 0 & 61.8 & 0.847 & 81.9 & 72.5 \\
    & Trial 5 & 33.1 & 39.8 & 0 & 22.6 & 0.919 & 63.4 & 48.3 \\
    & Trial 6 & 44.7 & 54.1 & 0 & -9.8 & 0.885 & 50.2 & 35.8 \\
    & Trial 7 & 69.1 & 75.4 & 0 & -64.0 & 0.897 & 27.4 & 17.0 \\
    & Trial 8 & 35.1 & 37.9 & 0 & 25.3 & 0.754 & 63.5 & 47.3 \\
    & Trial 9 & 15.5 & 25.1 & 0 & 55.8 & 0.723 & 79.4 & 69.4 \\
    & Trial 10 & 29.3 & 43.2 & 0 & 16.7 & 0.881 & 63.0 & 50.1 \\
    & Trial 11 & 26.3 & 30.9 & 1 & 40.5 & 0.781 & 71.3 & 58.3 \\
    & Trial 12 & 34.3 & 39.4 & 0 & 22.9 & 0.770 & 63.1 & 47.5 \\
    & Trial 13 & 35.4 & 46.5 & 0 & 8.5 & 0.786 & 58.5 & 42.2 \\
    & Trial 14 & 39.1 & 45.1 & 0 & 10.8 & 0.745 & 57.7 & 41.0 \\
    & Trial 15 & 39.4 & 44.6 & 0 & 11.9 & 0.825 & 57.9 & 40.9 \\
    \hline
    \multirow{4}{*}{Pedestrian} & Trial 11 & 5.9 & 5.5 & 0 & 88.6 & 0.640 & 94.3 & 91.4 \\
    & Trial 12 & 5.9 & 4.1 & 0 & 90.2 & 0.519 & 95.0 & 92.3 \\
    & Trial 13 & 6.9 & 4.5 & 0 & 88.7 & 0.670 & 94.3 & 89.9 \\
    & Trial 14 & 7.8 & 6.3 & 0 & 86.0 & 0.650 & 92.9 & 89.0 \\
    & Trial 15 & 6.3 & 6.7 & 0 & 87.0 & 0.666 & 93.5 & 90.0 \\
    \shline
\end{tabular}

    \label{tab:exp_results_system_b_day}
\end{table*}

\begin{table*}[ht]
    \renewcommand{\arraystretch}{1.2}
    \caption{Experimental Results for \textbf{System B} in \textbf{Night}}
    \small
    \setlength{\tabcolsep}{4pt}
    \centering
    \begin{tabular}{c|c|c|c|c|c|c|c|c}
    \shline
    Category & Trial & FP Rate $(\%)\downarrow$ & FN Rate $(\%)\downarrow$ & IDS $\downarrow$ & MOTA $(\%)\uparrow$ & MOTP (m)$\downarrow$ & IDF1 $(\%)\uparrow$ & HOTA $(\%)\uparrow$ \\
    \shline
    \multirow{12}{*}{Vehicle} & Trial 3 & 41.7 & 35.2 & 1 & 25.8 & 0.817 & 43.9 & 35.7 \\
    & Trial 4 & 34.8 & 24.1 & 1 & 43.7 & 0.846 & 46.2 & 41.7 \\
    & Trial 5 & 41.7 & 45.4 & 0 & 9.8 & 0.794 & 56.4 & 40.3 \\
    & Trial 6 & 28.7 & 38.9 & 0 & 25.5 & 0.883 & 65.8 & 52.1 \\
    & Trial 7 & 65.2 & 65.2 & 0 & -30.4 & 1.090 & 34.8 & 21.1 \\
    & Trial 8 & 31.1 & 32.2 & 0 & 36.1 & 0.790 & 68.3 & 52.4 \\
    & Trial 9 & 22.1 & 32.8 & 0 & 40.0 & 0.798 & 72.1 & 60.0 \\
    & Trial 10 & 10.2 & 20.8 & 0 & 66.1 & 0.793 & 84.2 & 77.0 \\
    & Trial 11 & 31.8 & 43.3 & 0 & 16.2 & 0.808 & 61.9 & 48.2 \\
    & Trial 12 & 31.3 & 37.5 & 0 & 27.3 & 0.786 & 65.4 & 50.4 \\
    & Trial 13 & 22.3 & 31.4 & 0 & 42.1 & 0.803 & 72.9 & 58.6 \\
    & Trial 14 & 33.7 & 39.0 & 0 & 23.9 & 0.770 & 63.5 & 46.9 \\
    & Trial 15 & 27.7 & 34.5 & 0 & 34.2 & 0.839 & 68.7 & 53.1 \\
    \hline
    \multirow{3}{*}{Pedestrian} & Trial 11 & 47.2 & 0.0 & 0 & 52.8 & 0.591 & 81.8 & 60.2 \\
    & Trial 12 & 45.9 & 0.0 & 0 & 54.1 & 0.603 & 84.2 & 62.7 \\
    & Trial 13 & 47.5 & 0.0 & 0 & 52.5 & 0.588 & 82.1 & 60.1 \\
    & Trial 14 & 45.7 & 0.0 & 0 & 54.3 & 0.551 & 84.5 & 62.8 \\
    \shline
\end{tabular}
    \label{tab:exp_results_system_b_night}
\end{table*}

\begin{table*}[ht]
    \renewcommand{\arraystretch}{1.2}
    \caption{Experimental Results for \textbf{System C} in \textbf{Daytime}}
    \small
    \setlength{\tabcolsep}{4pt}
    \centering
    \begin{tabular}{c|c|c|c|c|c|c|c|c}
    \shline
    Category & Trial & FP Rate $(\%)\downarrow$ & FN Rate $(\%)\downarrow$ & IDS $\downarrow$ & MOTA $(\%)\uparrow$ & MOTP (m)$\downarrow$ & IDF1 $(\%)\uparrow$ & HOTA $(\%)\uparrow$ \\
    \shline
    \multirow{12}{*}{Vehicle} & Trial 3 & 1.4 & 4.5 & 1 & 93.0 & 0.458 & 55.1 & 60.8 \\
    & Trial 4 & 6.6 & 22.1 & 0 & 66.9 & 0.496 & 85.0 & 80.4 \\
    & Trial 5 & 19.4 & 46.5 & 1 & 9.2 & 0.705 & 34.2 & 37.5 \\
    & Trial 6 & 33.7 & 56.1 & 0 & -18.4 & 0.514 & 52.8 & 42.2 \\
    & Trial 7 & 19.4 & 41.5 & 0 & 23.4 & 0.645 & 67.8 & 58.9 \\
    & Trial 8 & 24.5 & 48.4 & 1 & 3.6 & 0.596 & 36.2 & 36.7 \\
    & Trial 9 & 10.9 & 27.2 & 0 & 55.5 & 0.875 & 80.1 & 73.3 \\
    & Trial 10 & 16.1 & 35.7 & 1 & 36.4 & 0.584 & 38.2 & 41.5 \\
    & Trial 11 & 22.4 & 38.6 & 0 & 28.9 & 0.583 & 68.5 & 57.6 \\
    & Trial 12 & 12.6 & 26.4 & 0 & 55.8 & 0.678 & 79.9 & 71.9 \\
    & Trial 13 & 21.4 & 45.9 & 1 & 11.6 & 0.627 & 57.4 & 48.9 \\
    & Trial 14 & 18.3 & 32.9 & 2 & 41.0 & 0.786 & 57.4 & 50.3 \\
    & Trial 15 & 21.6 & 38.0 & 2 & 29.9 & 0.657 & 61.6 & 47.6 \\
    \hline
    \multirow{6}{*}{Pedestrian} & Trial 11 & 6.8 & 16.9 & 4 & 73.3 & 0.422 & 40.7 & 48.8 \\
    & Trial 12 & 7.2 & 17.1 & 4 & 72.8 & 0.429 & 60.4 & 63.3 \\
    & Trial 13 & 7.3 & 28.3 & 7 & 54.5 & 0.386 & 26.8 & 37.9 \\
    & Trial 14 & 6.9 & 21.3 & 2 & 67.3 & 0.435 & 41.7 & 49.5 \\
    & Trial 15 & 7.1 & 18.1 & 3 & 71.6 & 0.466 & 51.9 & 57.0 \\
    \shline
\end{tabular}

    \label{tab:exp_results_system_c_day}
\end{table*}

\begin{table*}[ht]
    \renewcommand{\arraystretch}{1.2}
    \caption{Experimental Results for \textbf{System C} in \textbf{Night}}
    \small
    \setlength{\tabcolsep}{4pt}
    \centering
    \begin{tabular}{c|c|c|c|c|c|c|c|c}
    \shline
    Category & Trial & FP Rate $(\%)\downarrow$ & FN Rate $(\%)\downarrow$ & IDS $\downarrow$ & MOTA $(\%)\uparrow$ & MOTP (m)$\downarrow$ & IDF1 $(\%)\uparrow$ & HOTA $(\%)\uparrow$ \\
    \shline
    \multirow{11}{*}{Vehicle} & Trial 3 & 2.3 & 5.0 & 1 & 91.7 & 0.519 & 55.3 & 60.6 \\
    & Trial 4 & 0.0 & 18.1 & 0 & 77.9 & 0.506 & 90.0 & 90.5 \\
    & Trial 5 & 18.3 & 46.1 & 1 & 10.8 & 0.617 & 38.5 & 40.6 \\
    & Trial 6 & 20.6 & 50.7 & 0 & -2.9 & 0.560 & 60.8 & 53.6 \\
    & Trial 7 & 19.8 & 37.0 & 0 & 33.0 & 0.460 & 70.5 & 60.4 \\
    & Trial 8 & 26.3 & 30.6 & 1 & 40.4 & 0.523 & 39.1 & 39.2 \\
    & Trial 9 & 7.8 & 24.7 & 0 & 62.2 & 0.767 & 82.9 & 77.8 \\
    & Trial 10 & 8.6 & 20.8 & 1 & 66.4 & 0.568 & 48.8 & 52.3 \\
    & Trial 11 & 27.3 & 42.7 & 0 & 18.6 & 0.596 & 64.1 & 53.5 \\
    & Trial 12 & 13.5 & 26.9 & 2 & 53.7 & 0.559 & 60.8 & 57.8 \\
    & Trial 13 & 21.3 & 37.9 & 1 & 30.6 & 0.566 & 66.2 & 54.3 \\
    & Trial 14 & 13.2 & 30.8 & 1 & 47.2 & 0.655 & 47.2 & 46.5 \\
    & Trial 15 & 12.8 & 24.8 & 1 & 53.4 & 0.522 & 78.8 & 58.0 \\
    \hline
    \multirow{3}{*}{Pedestrian} & Trial 11 & 58.8 & 29.3 & 8 & 23.5 & 0.380 & 19.2 & 19.6 \\
    & Trial 12 & 30.9 & 29.3 & 5 & 39.7 & 0.400 & 21.4 & 27.8 \\
    & Trial 13 & 0.0 & 54.1 & 9 & -20.7 & 0.376 & 20.4 & 33.7 \\
    \shline
\end{tabular}
    \label{tab:exp_results_system_c_night}
\end{table*}

\begin{table*}[ht]
    \renewcommand{\arraystretch}{1.2}
\centering
\caption{The Mean Results of System A, System B, and System C over all trials}
\label{tab:overall_results}
\small
\setlength{\tabcolsep}{4pt}
\centering
\begin{tabular}{c|c|c|c|c|c|c|c}
    \shline
    & FP Rate $(\%)\downarrow$ & FN Rate $(\%)\downarrow$ & IDS $\downarrow$ & MOTA $(\%)\uparrow$ & MOTP (m)$\downarrow$ & IDF1 $(\%)\uparrow$ & HOTA $(\%)\uparrow$ \\
    \shline
\multicolumn{8}{c}{\textbf{Vehicle detection in daytime}} \\
\hline
System A & $0.0$ & $2.1$  & $0.0$  & $92.0$ & $0.402$ & $96.3$ & $96.3$ \tabularnewline
\hline
System B & $34.8$ & $41.4$ & $0.2$ & $17.9$ & $0.813$ & $60.4$ & $46.9$ \tabularnewline
\hline
System C & $16.9$ & $33.8$ & $0.7$ & $37.2$ & $0.641$ & $62.4$ & $56.8$ \tabularnewline
\hline
\multicolumn{8}{c}{\textbf{Vehicle detection in night}} \\
\hline
System A & $0.0$ & $2.0$ & $0.0$ & $97.8$ & $0.365$ & $99.0$ & $98.5$ \tabularnewline
\hline
System B & $32.5$ & $37.0$ & $0.6$ & $27.7$ & $0.832$ & $61.9$ & $49.0$ \tabularnewline
\hline
System C & $14.8$ & $30.8$ & $1.0$ & $44.7$ & $0.570$ & $61.7$ & $58.0$ \tabularnewline
\shline
\multicolumn{8}{c}{\textbf{Pedestrian detection in daytime}} \\
\hline
System A & $7.8$ & $0.0$ & $0.0$ & $92.2$ & $0.367$ & $96.1$ & $92.4$ \tabularnewline
\hline
System B & $6.6$ & $5.4$ & $0.0$ & $88.1$ & $0.629$ & $94.0$ & $90.5$ \tabularnewline
\hline
System C & $8.1$ & $20.8$ & $3.8$ & $66.5$ & $0.440$ & $44.0$ & $50.9$ \tabularnewline
\hline
\multicolumn{8}{c}{\textbf{Pedestrian detection in night}} \\
\hline
System A & $3.1$ & $0.0$ & $0.0$ & $96.9$ & $0.348$ & $98.5$ & $96.9$ \tabularnewline
\hline
System B & $46.6$ & $0.0$ & $0.0$ & $53.4$ & $0.583$ & $83.1$ & $61.4$ \tabularnewline
\hline
System C & $29.9$ & $37.6$ & $7.3$ & $61.5$ & $0.385$ & $20.3$ & $27.0$ \tabularnewline
\shline
\end{tabular}
\end{table*}
\subsubsection{Latency Measurement Results}

We conducted two trials for each perception system to measure the latency. This was achieved with the method described in Section \ref{sss:latency-measurement} by driving vehicles back and forth along predetermined routes with constant speed. The first trial involved driving in the north-south direction, and the second in the east-west direction. Each trial consisted of five rounds of driving. 
The results of these measurements are presented in Table \ref{tab:latency}. System A demonstrates a latency of 48 milliseconds, which is well within the acceptable range for cooperative driving tasks, showcasing its efficiency in real-time responsiveness. System B, with a latency of 145 ms, also fits within a threshold suitable for a broad spectrum of applications. In contrast, System C exhibits a latency of 1715 ms, this level of latency falls short of the requirements for many real-time applications, highlighting a critical area for improvement.

As observed in Table \ref{tab:latency}, the latency measured during the North-South trials is close to the latency observed in the East-West trials for both systems, with a small proportional variance. This consistency provides validation for our measurements, reinforcing our confidence in the reliability of the latency estimates. The latency measurements obtained here are utilized in the subsequent assessments described in the following sections.

\subsubsection{Qualitative Results Review}

The qualitative analysis is predicated on a visual comparison between the detected trajectories and the ground truth data, as illustrated in Figure \ref{fig:detection_results}. In this figure, each row corresponds to results from the same system, while each column represents results from the same trial across different systems. The detection results are indicated by red dots, and the ground truth data is represented by blue dots.

The first column of the figure displays the detection results from Trial 3, and the second column presents the results from the same trial conducted at night. Notably, all systems perform relatively well under varying light conditions. This performance is particularly surprising for System B and C, which are image-based detection systems, implying that its underlying algorithm effectively handles the challenges posed by low light conditions.
The third column of the figure presents the pedestrian detection results. In this scenario, System A and B provide better localization results. 
The fourth column in the figure shows the results from Trial 15, which involves two vehicles approaching the intersection from perpendicular directions. All systems demonstrate capability in tracking multiple vehicles, with results aligning closely with the ground truth data.

In summary, the visualized results indicate that all Systems perform well in general, accurately detecting and localizing road users in various conditions.

\subsubsection{Quantitative Results}

Detailed quantitative results for each trial are outlined in six separate tables. Table \ref{tab:exp_results_system_a_day} presents the results for daytime trials conducted using System A, Table \ref{tab:exp_results_system_b_day} displays same results for System B and \ref{tab:exp_results_system_c_day} for System C. The performance of all systems under night conditions are presented separately in Table \ref{tab:exp_results_system_a_night} for System A, Table \ref{tab:exp_results_system_b_night} for System B and Table \ref{tab:exp_results_system_c_night} for System C.
In each table, each row corresponds to a specific trial. It is important to note that for Trials 11 to 15, the metrics were calculated separately for vehicles and pedestrians, despite the simultaneous detection of both during the trials. 


Insights drawn from Tables \ref{tab:exp_results_system_a_day} and \ref{tab:exp_results_system_a_night} reveal that System A demonstrates notable efficacy, particularly in the domains of vehicle detection, tracking, and localization. This system records impressive scores in MOTA, IDF1, and HOTA across most vehicle trials conducted during both daytime and nighttime conditions. However, its performance in pedestrian detection is marginally less effective, as evidenced by a higher incidence of false positives. This observation aligns with expectations, considering that pedestrians, being smaller and consequently less prominent in the LiDAR point cloud, pose a greater challenge for accurate detection.

Tables \ref{tab:exp_results_system_b_day} and \ref{tab:exp_results_system_b_night} suggest that System B's performance is notably inferior to that of System A. Despite the visually satisfactory outputs depicted in Figure \ref{fig:detection_results}, System B's quantitative results fall short of expectations. This divergence is largely due to the system's latency variations, characterized by a high standard deviation, as detailed in Table \ref{tab:latency}. Such fluctuations play a critical role in contributing to localization errors. The reason for this is that when calculating localization errors, we match the detection point with the ground truth data from the same time frame. However, visually, all detection points are plotted on the map irrespective of their detection time, masking the impact of latency variation. These factors predominantly affect the system's localization accuracy, resulting in comparable performance levels during both daytime and nighttime trials. However, System B excels in pedestrian detection, demonstrating significantly improved results. This enhancement is largely because the slower movement of pedestrians lessens the impact of latency variation on their localization.

Tables \ref{tab:exp_results_system_c_day} and \ref{tab:exp_results_system_c_night} present the performance metrics of System C. These tables reveal that System C outperforms System B in vehicle detection, primarily due to its reduced latency variation. However, there still exists a noticeable performance gap when compared to System A. Regarding pedestrian detection, System C shows less optimal results, with this shortcoming being particularly pronounced during daytime trials.

Table \ref{tab:overall_results} provides a comprehensive summary of the mean results from all trials conducted with Systems A, B, and C, encompassing various conditions and targeting both vehicles and pedestrians. The data consolidated in this table corroborate our earlier discussions, clearly indicating that the overall performance of Systems B and C falls short of System A's efficiency. This disparity is particularly evident in image-based systems, which appear to struggle with latency issues more than lidar-based systems. This is likely due to more intensive processing tasks required by image sensors. Addressing this latency challenge is evidently a critical area for improvement in these systems.  

Furthermore, Table \ref{tab:overall_results} sheds light on the influence of lighting conditions on the three systems. It becomes evident that lighting conditions do not affect System A, a LiDAR-based system, demonstrating its robustness in varied environments. In contrast, Systems B and C exhibit noticeable differences in pedestrian detection performance under varying lighting conditions. Between System B and System C, System C demonstrates superior vehicle detection, and System B shows better performance in pedestrian detection. 

It is also significant to note that the Mean Object Tracking Precision (MOTP) metric for System A is comparatively lower. A lower MOTP value indicates superior localization accuracy, an area where LiDAR-based detection systems, such as System A, are widely recognized to excel. This advantage is primarily due to the inherent capabilities of LiDAR technology in precisely mapping object locations, a feat that often challenges image-based detection systems.

\subsection{Discussion}
From the results discussed above, we observed a notable performance gap between the LiDAR-based and image-based methods. This disparity is primarily attributed to the stringent 1.5-meter distance threshold applied during our evaluation, a standard derived from the SAE2945 \cite{sae2016board}, which stipulates 1.5 meters as the maximum acceptable localization error for autonomous vehicles using perception results. Figure \ref{fig:threshold-analysis} illustrates the false positive rate and false negative rate of all three systems when different thresholds are applied. The figure reveals that the FP and FN rates of the three systems asymptotically approach a lower value. Notably, System A, the LiDAR-based system, reaches this lower value with a much smaller localization threshold. This finding indicates that while image-based detection systems do accurately detect road objects, their higher localization errors result in many correctly detected objects being categorized as false positives when evaluated against the 1.5-meter threshold. Consequently, these results suggest a clear area for improvement in image-based systems: enhancing the accuracy of localization in future iterations.

\begin{figure}[htb]
    \centering
    \begin{subfigure}{.5\linewidth}
        \centering
        \includegraphics[width=\linewidth]{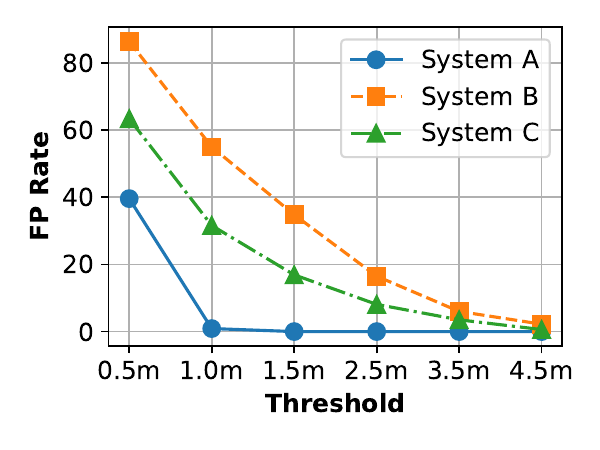} 
        \caption{FP analysis}
        \label{fig:threshold-FP}
    \end{subfigure}%
    \begin{subfigure}{.5\linewidth}
        \centering
        \includegraphics[width=\linewidth]{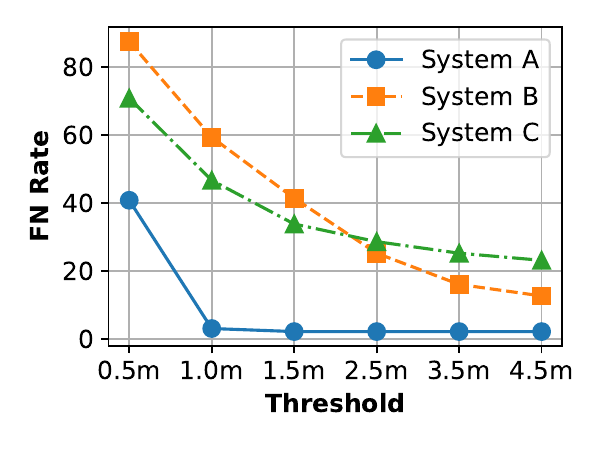} 
        \caption{FN analysis}
        \label{fig:threshold-FN}
    \end{subfigure}
    \caption{Evaluation results under different distance threshold.}
    \label{fig:threshold-analysis}
\end{figure}

\section{Conclusion}

In this paper, we have presented a systematic evaluation methodology for roadside perception systems, encompassing measurement techniques, metrics selection, and experimental trial design. We conducted comprehensive experiments on Three representative off-the-shelf roadside perception systems (referred to as System A and System B) within a controlled testing facility. Our experiments evaluated these systems under varying conditions and across diverse traffic scenarios involving vehicles and pedestrians.

The results showcase the viability of our proposed evaluation methodology in enabling standardized, comparative analysis of different roadside perception systems. We were able to quantitatively analyze numerous performance metrics for both systems and identify their respective strengths and weaknesses. 

Our study and proposed methodology make significant headway towards establishing standardized benchmarks and evaluation practices for roadside perception systems. We believe widespread adoption of systematic evaluation approaches like ours will greatly benefit the research and development of infrastructure-based perception for autonomous driving. Standardized benchmarks will allow for more rigorous testing, direct comparison between products, identification of limitations, and targeted improvements.

\appendix

\subsection{Analysis on Impact of Slight Speed Deviations}
\label{appendix:1}
In the practical implementation of the latency measurement described in Section \ref{sss:latency-measurement}, the driver endeavors to maintain a constant speed. However, it is inevitable that the speed will exhibit minor fluctuations. This section offers a brief analysis of the effects of such speed variability. Adhering to the mathematical notation in Section \ref{sss:latency-measurement}, when the speed is variable, we denote the speed $v(t)$ as a continuous random process. Define $G(t)=\int_0^tv(t)dt$. Incorporating this into formula (\ref{eq:1}), we obtain:
$$\int_0^{\mathbf{t_2}-\mathbf{l}}v(t)dt + e_1 + \mathbf{e_2} = \int_0^\mathbf{t_1}v(t)dt$$
$$e_1 + \mathbf{e_2} = \int_0^{\mathbf{t_1} - \mathbf{t_2} + \mathbf{l}}v(t)dt$$
Given that $t_1-t_2+l$ is a small quantity, the following approximation is viable:
$$\int_0^{\mathbf{t_1} - \mathbf{t_2} + \mathbf{l}}v(t)dt \approx (v_0 + \bm{\mu}) (\mathbf{t_1} - \mathbf{t_2} + \mathbf{l})$$

Here, $\bm{\mu}$ represents a zero-mean random variable that captures the variability of speed $v(t)$ over a brief duration. Then, the variation in speed can be explicitly formulated in the following equation:
\begin{equation}
    (v_0 + \bm{\mu})(\mathbf{t_1} - \mathbf{t_2} + \bm{l}) = e_1 + \mathbf{e_2}
    \label{eq:appendix1}
\end{equation}
Taking the expectation on both sides confirms that (\ref{eq:4}) remains valid:
$$
E[\bm{\tau}] = E[\bm{l}] - \frac{e_1}{v_0}
$$

\subsection{Variance Analysis of Latency Measurement}
\label{appendix:2}
This section presents a variance analysis for the method outlined in Section \ref{sss:latency-measurement}. Employing equation (\ref{eq:appendix1}), we obtain:
\begin{equation}
    \bm{\tau} = \bm{l} - \frac{e_1 + \mathbf{e_2}}{v_0 + \bm{\mu}}
\end{equation}
When calculating the variance of both sides, we derive:
$$Var(\bm{\tau}) = Var(\bm{l}) + Var\left(\frac{e_1}{v_0 + \bm{\mu}}\right) + Var\left(\frac{\mathbf{e_2}}{v_0 + \bm{\mu}}\right)$$
This leads to the following approximation, assuming minor fluctuations in $\bm{\mu}$:
\begin{equation}
Var(\bm{\tau}) \approx Var(\bm{l}) + \frac{1}{v_0^2}Var(\mathbf{e_2})
    \label{eq:appendix2}
\end{equation}
Equation (\ref{eq:appendix2}) indicates that the variability of the measured latency $\bm{\tau}$ is influenced by both the latency and localization errors. It also suggests that choosing a higher traveling speed $v_0$ during the experiment can reduce the variance, bringing it closer to the true variance of the latency. Furthermore, as $Var(\bm{l})$ remains constant across different locations, the measured latency variance can serve as an indicator of positional accuracy in varying locations.

\subsection{Variance Analysis of Position Measurement}
\label{appendix:3}
This section provides a concise analysis of the variance associated with the estimator $\bm{\Tilde{e_d}}$, as utilized in Section \ref{sss:positioning-error}. Employing equation (\ref{eq:appendix1}), we derive:
$$Var[\bm{\Tilde{e_d}}] = Var[e_1 + \bm{e_2}] + Var[(v_0+\bm{\mu})(E[\bm{l}] - \bm{l})]$$
$$= Var[\bm{e_2}] + (v_0^2 + Var[\bm{\mu}])Var[\bm{l}]$$
Assuming a small variance of the term $\bm{\mu}$, the variance of the estimator $Var[\bm{\Tilde{e_d}}]$ can be approximated as:
\begin{equation}
    Var[\bm{\Tilde{e_d}}] \approx Var[\bm{e_2}] + v_0^2Var[\bm{l}]
    \label{eq:appendix3}
\end{equation}
This equation highlights the compounded effect of latency and positioning errors, echoing the findings in equation (\ref{eq:appendix2}). It also reveals an intrinsic relationship between (\ref{eq:appendix2}) and (\ref{eq:appendix3}), characterized by a factor of $v_0^2$.

\bibliographystyle{ieeetr}
\bibliography{bibtex}

\section{Biography}
\begin{IEEEbiography}[{\includegraphics[width=1in,height=1.25in,clip,keepaspectratio]{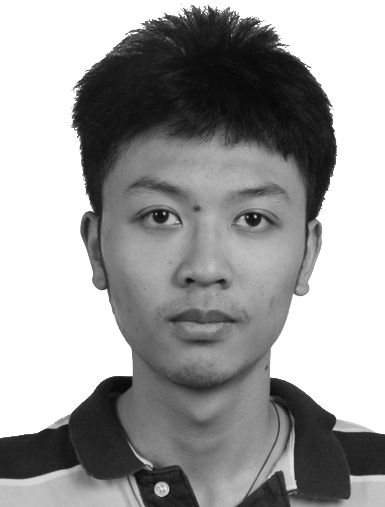}}]{Rusheng Zhang}
 received the B.E. degree in micro electrical mechanical system and second B.E. degree in Applied Mathematics from Tsinghua University, Beijing, in 2013. He received an M.S. and phD degree in electrical and computer engineering from Carnegie Mellon University, in 2015, 2019 respectively. His research areas include artificial intelligence, cooperative driving, cloud computing and vehicular networks.
\end{IEEEbiography}
\begin{IEEEbiography}[{\includegraphics[width=1in,height=1.25in,clip,keepaspectratio]{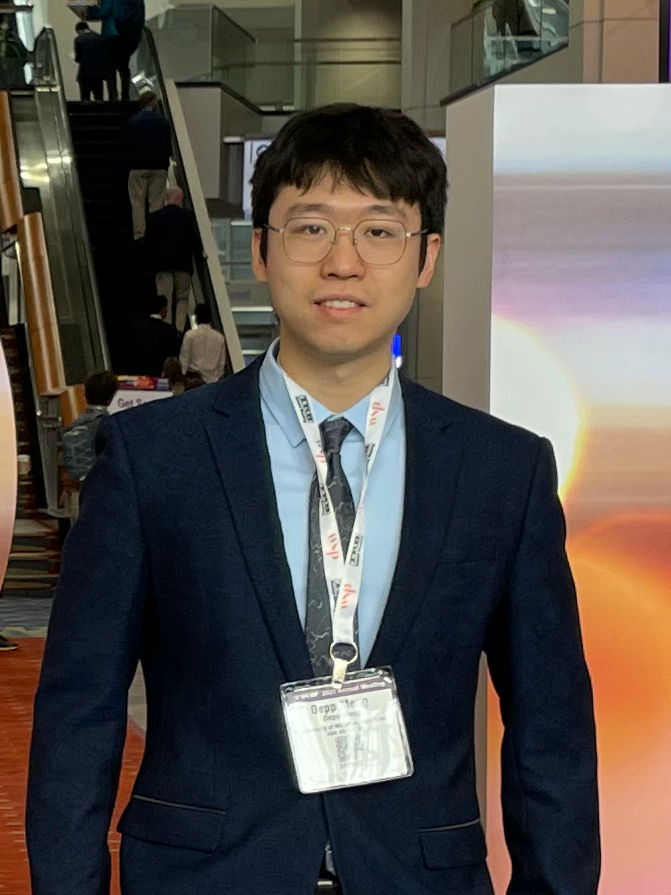}}]{Depu Meng} (Member, IEEE) is a Post Doctoral Research Fellow at the Department of Civil and Environmental Engineering, University of Michigan. He received his B. E. degree from the Department of Electrical Engineering and Information Science at the University of Science and Technology of China in 2018. He received his Ph. D. degree from the Department of Automation at the University of Science and Technology of China. His research interests include computer vision and autonomous driving systems.
\end{IEEEbiography}

\begin{IEEEbiography}[{\includegraphics[width=1in,height=1.25in,clip,keepaspectratio]{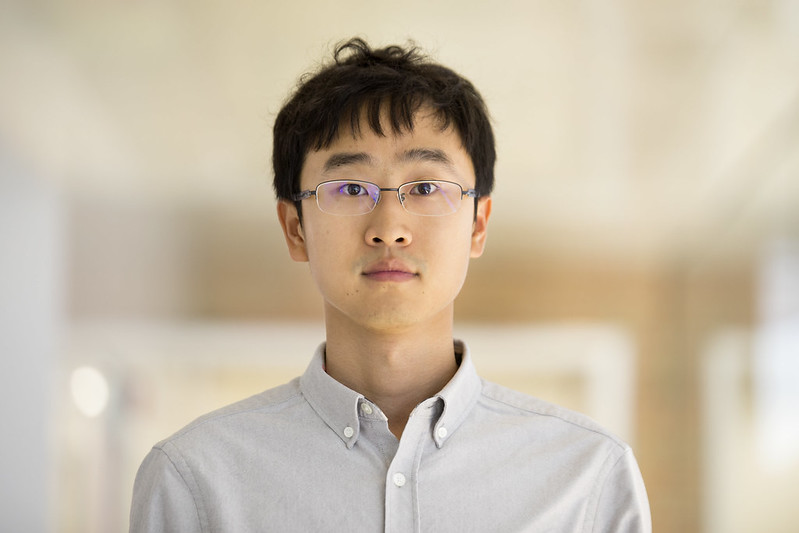}}]{Shengyin (Sean) Shen} works as a Research Engineer in the Engineering Systems Group at the University of Michigan Transportation Research Institute (UMTRI). Sean holds an MS degree in Civil and Environmental Engineering from the University of Michigan, Ann Arbor, and an MS degree in Electrical Engineering from the University of Bristol, UK. He also earned a BS degree from Beijing University of Posts and Telecommunications, China. Sean's research interests are primarily focused on cooperative driving automation and related applications that use roadside perception, edge-cloud computing, and V2X communications to accelerate the deployment of automated vehicles. He has extensive experience in implementation of large-scale deployments, such as the Safety Pilot Model Deployment (SPMD), Ann Arbor Connected Vehicle Testing Environment (AACVTE), and Smart Intersection Project. Moreover, he has been involved in many research projects funded by public agencies such as USDOT, USDOE, and companies such as Crash Avoidance Metric Partnership (CAMP), Ford Motor Company, and GM Company, among others.
\end{IEEEbiography}
\begin{IEEEbiography}[{\includegraphics[width=1in,height=1.25in,clip,keepaspectratio]{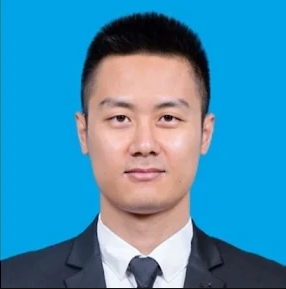}}]{Tinghan Wang} received the B.Tech. and Ph.D. degrees from Tsinghua
University, Beijing, China, in 2016 and 2022. He is currently a
research fellow with the Mechanical Engineering Department, the
University of Michigan. His research interests include automated
vehicle evaluation, cooperative driving, and end-to-end self-driving
based on deep reinforcement learning.
\end{IEEEbiography}
\begin{IEEEbiography}[{\includegraphics[width=1in,height=1.25in,clip,keepaspectratio]{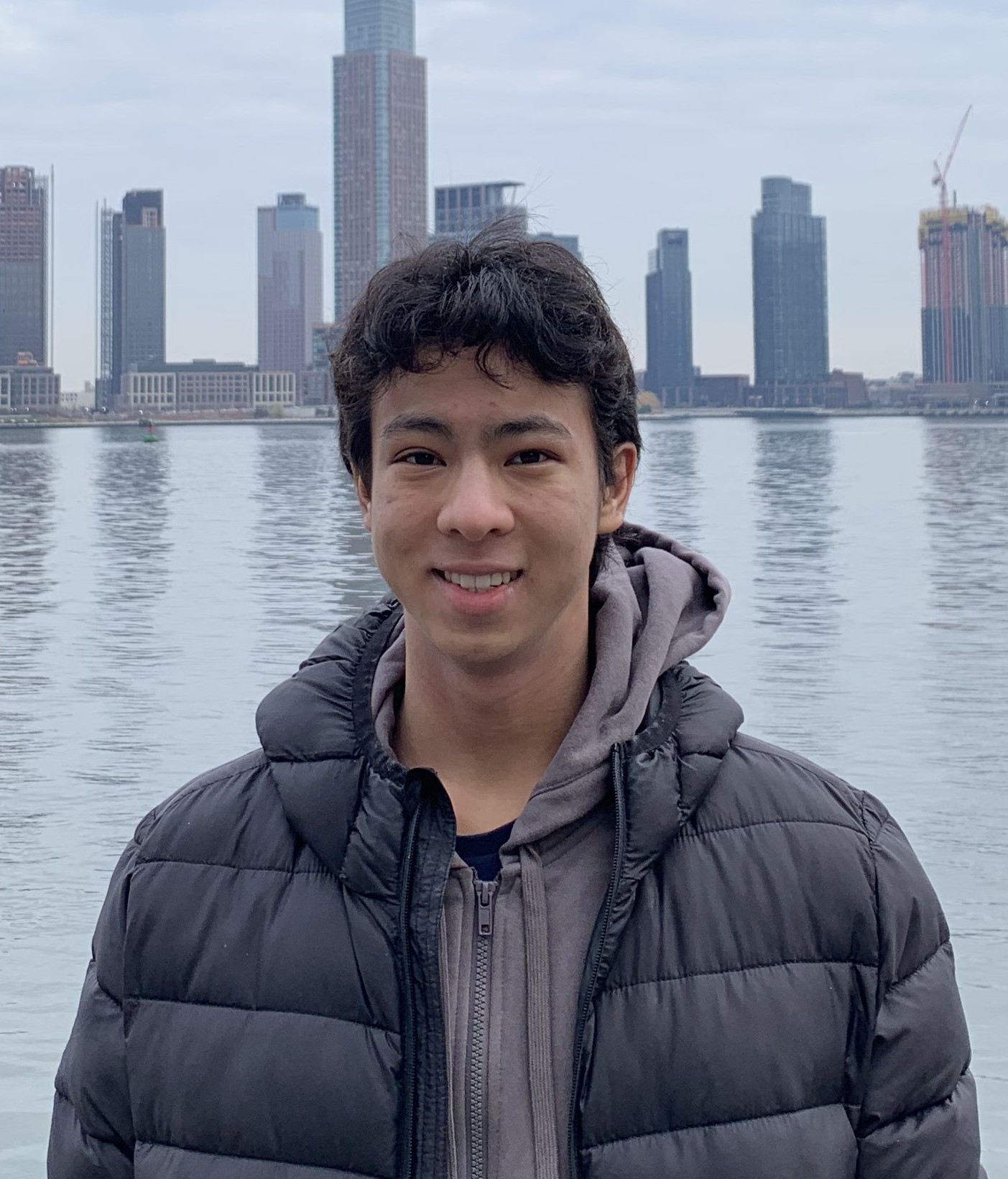}}]{Tai Karir} is a senior at Huron High School. He is a member of the soccer team and also enjoys playing piano and reading. His research interests include software engineering, data analysis and machine learning.
\end{IEEEbiography}
\begin{IEEEbiography}[{\includegraphics[width=1in,height=1.25in,clip,keepaspectratio]{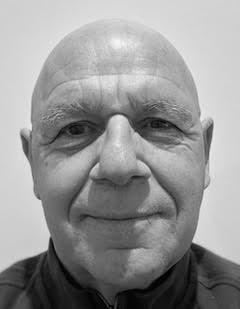}}]{Michael Maile} received the Dipl. Phys. Degree from the University of Ulm in 1986. He is currently a Principal at Ivie Communications. His research areas include Sensor Fusion and Localization for Automated Vehicles and Infrastructure-Vehicle communication based safety applications. 
\end{IEEEbiography}
\begin{IEEEbiography}[{\includegraphics[width=1in,height=1.25in,clip,keepaspectratio]{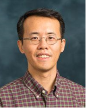}}]{Henry X. Liu}
(Member, IEEE) received the bachelor's degree in automotive engineering from Tsinghua University, China, in 1993, and the PhD. degree in civil and environment engineering from the University of Wisconsin-Madison in 2000. He is currently a professor in the Department of Civil and Environmental Engineering and the Director of Mcity at the University of Michigan, Ann Arbor. He is also a Research Professor at the University of Michigan Transportation Research Institute and the Director for the Center for Connected and Automated Transportation (USDOT Region 5 University Transportation Center). From August 2017 to August 2019, Prof. Liu served as DiDi Fellow and Chief Scientist on Smart Transportation for DiDi Global, Inc., one of the leading mobility service providers in the world. Prof. Liu conducts interdisciplinary research at the interface of transportation engineering, automotive engineering, and artificial intelligence. Specifically, his scholarly interests concern traffic flow monitoring, modeling, and control, as well as testing and evaluation of connected and automated vehicles. Prof. Liu is the managing editor of Journal of Intelligent Transportation Systems.
\end{IEEEbiography}



\end{document}